\documentclass[nohyperref]{article}

\usepackage{microtype}
\usepackage{graphicx}
\usepackage{subfigure}
\usepackage{booktabs} % for professional tables

\usepackage{hyperref}

\usepackage[accepted]{icml2023}

\usepackage{amsmath}
\usepackage{amssymb}
\usepackage{mathtools}
\usepackage{amsthm}
\usepackage{bbm}
\usepackage{algorithm}
\usepackage{bm}
\usepackage{xcolor}   % colors
\usepackage{multirow}
\usepackage{thmtools, thm-restate}
\usepackage[capitalize,noabbrev]{cleveref}

\theoremstyle{plain}
\newtheorem{theorem}{Theorem}[section]

\newtheorem{lemma}[theorem]{Lemma}

\theoremstyle{definition}

\theoremstyle{remark}

\usepackage[textsize=tiny]{todonotes}

\icmltitlerunning{Anti-Robust Weighted Regularization}

\begin{document}

% The \author macro works with any number of authors. There are two commands
% used to separate the names and addresses of multiple authors: \And and \AND.
%
% Using \And between authors leaves it to LaTeX to determine where to break the
% lines. Using \AND forces a line break at that point. So, if LaTeX puts 3 of 4
% authors names on the first line, and the last on the second line, try using
% \AND instead of \And before the third author name.

\twocolumn[
\icmltitle{Improving Adversarial Robustness \\ 
by Putting More Regularizations on Less Robust Samples}



\icmlsetsymbol{equal}{*}

\begin{icmlauthorlist}
    
\icmlauthor{Dongyoon Yang}{snu}
\icmlauthor{Insung Kong}{snu}
\icmlauthor{Yongdai Kim}{snu}
\end{icmlauthorlist}

\icmlaffiliation{snu}{Department of Statistics, Seoul National University, Seoul, Republic of Korea}
\icmlcorrespondingauthor{\text{Yongdai Kim}}{\text{ydkim0903@gmail.com}}

\icmlkeywords{Machine Learning, ICML}

\vskip 0.3in
]



\printAffiliationsAndNotice{}  % leave blank if no need to mention equal contribution

\begin{abstract}
Adversarial training, which is to enhance robustness against adversarial attacks, has received much attention because it is easy to generate human-imperceptible perturbations of data to deceive a given deep neural network. 
In this paper, we propose a new adversarial training algorithm that is theoretically well motivated and empirically superior to other existing algorithms. 
A novel feature of the proposed algorithm is to apply more regularization to data vulnerable to adversarial attacks than other existing regularization algorithms do. Theoretically, we show that our algorithm can be understood as an algorithm of minimizing the regularized empirical risk motivated from a newly derived upper bound of the robust risk. 
Numerical experiments illustrate that our proposed algorithm improves the generalization (accuracy on examples) and robustness (accuracy on adversarial attacks) simultaneously to achieve the state-of-the-art performance.
\end{abstract}

\section{Introduction}

It is easy to generate human-imperceptible perturbations that put prediction of a deep neural network (DNN) out. Such perturbed samples are called \textit{adversarial examples} \citep{szegedy2014intriguing} and algorithms for generating adversarial examples are called \textit{adversarial attacks}. 
%Adversarial attacks can be divided into two types -  white-box attacks \citep{goodfellow2015explaining, madry2018towards, carlini2017evaluating, croce2020minimally} and black-box attacks \citep{papernot2016transferability, papernot2017practical, chen2017zoo, ilyas2018blackbox, papernot2016science}. 
%In the white-box attack setting, an adversary has access to the model architecture and the value of parameters of a given DNN , while in the black-box attack setting, the adversary has no access to them but can access only to  outputs of the prediction model for given inputs.
It is well known that adversarial attacks can greatly reduce the accuracy of DNNs, for example from about 96\% accuracy on clean data to almost zero accuracy on adversarial examples \citep{madry2018towards}. This vulnerability of DNNs can cause serious security problems when DNNs are applied to  security critical applications \citep{kurakin2016adversarial,  jiang2019blackbox} such as medicine \citep{ma2021understanding, finlayson2019adversarial} and autonomous driving \citep{kurakin2016adversarial, deng2020analysis, morgulis2019fooling, li2020adaptive}.

Adversarial training, which is to enhance robustness against adversarial attacks,
has received much attention. Various adversarial training algorithms can be 
categorized into two types.
The first one is to learn prediction models by minimizing
the robust risk - the risk for adversarial examples.
PGD-AT \citep{madry2018towards} is the first of its kinds and
various modifications including \cite{zhang2020attacks, ding2020mma, zhang2021geometry} 
have been proposed since then.

The second type of adversarial training algorithms is to minimize the regularized risk which is
the sum of the empirical risk for clean examples and a regularized term related to adversarial robustness.
TRADES \citep{zhang2019theoretically} decomposes the robust risk into the sum
of the natural and boundary risks, where the first one is the risk for clean examples
and the second one is the remaining part, and replaces them to their upper bounds to have
the regularized risk. HAT \citep{rade2022reducing} modifies the regularization term of TRADES
by adding an additional regularization term based on helper samples. 

The aim of this paper is to develop a new adversarial training algorithm for DNNs, which is theoretically well motivated and empirically superior to other existing competitors. 
Our algorithm modifies the regularization term of TRADES \citep{zhang2019theoretically} to put more regularization on less robust samples. This new regularization term is motivated
by an upper bound of the boundary risk. 

Our proposed regularized term is similar to that used in MART \citep{wang2020improving}.
The two key differences are that (1) the objective function of MART consists of the sum of the robust risk and regularization term while ours consists of the sum of the natural risk and regularization term and
(2) our algorithm regularizes less robust samples more but MART regularizes less accurate samples more.
Note that our algorithm is theoretically well motivated from an upper bound of the robust risk 
but no such theoretical explanation of MART is available. In numerical studies, we demonstrate that
our algorithm outperforms MART as well as TRADES with significant margins.

%--------------------------------------------------------------
\subsection{Our Contributions}

We propose a new adversarial training algorithm. Novel features of our algorithm compared to other existing adversarial training  algorithms are that it is theoretically well motivated and empirically superior. Our contributions can be summarized as follows:
\begin{itemize}
\item We derive an upper bound of the robust risk for multi-classification problems
\item As a surrogate version of this upper bound, we propose
a new regularized risk.
\item We develop an adversarial training algorithm that
learns a robust prediction model by minimizing the proposed regularized risk.
\item By analyzing benchmark data sets, we show that our proposed algorithm is superior to other competitors in view of 
the generalization (accuracy on clean examples) and robustness (accuracy on adversarial examples) simultaneously to achieve the state-of-the-art performance.
\item We illustrate that our algorithm is helpful to improve the fairness of the prediction model in the sense that the error rates of each class become more similar compared to TRADES and HAT.
\end{itemize}
%--------------------------------------------------------------
\section{Preliminaries}
\label{sec2}
\subsection{Robust Population Risk} 
\label{sec2_1}
Let $\mathcal{X} \subset \mathbb{R}^d$ be the input space, $\mathcal{Y} = \left\{1, \cdots, C\right\}$
be the set of output labels and $f_{\bm{\theta}} : \mathcal{X} \rightarrow \mathbb{R}^{C}$ be the score function parameterized by the neural network parameters $\bm{\theta}$ (the vector of weights and biases) such that $\mathbf{p}_{\bm{\theta}}(\cdot|\bm{x}) =\operatorname{softmax}(f_{\bm{\theta}}(\bm{x}))$ is the vector of the conditional
class probabilities. Let $F_{\bm{\theta}}(\bm{x}) = \underset{c}{\operatorname{argmax}} [f_{\bm{\theta}}(\bm{x})]_c,$  
$\mathcal{B}_{p}(\bm{x}, \varepsilon) = \left\{\bm{x}' \in \mathcal{X} : \lVert \bm{x}- \bm{x}' \rVert_p \leq \varepsilon \right\}$ and $\mathbbm{1}(\cdot)$ be the indicator function. Let capital letters $\mathbf{X, Y}$ denote random variables or vectors and small letters $\bm{x}, y$ denote their realizations.

The robust population risk used in the adversarial training is defined as
\begin{equation}
\label{eqn1}
\mathcal{R}_{\text{rob}}(\theta):=\mathbb{E}_{\mathbf{(X,Y)}} \;  \underset{\mathbf{X'} \in \mathcal{B}_p(\mathbf{X}, \varepsilon) }{\max\;\;\;} \mathbbm{1}\left\{F_{\bm{\theta}}(\mathbf{X'}) \neq \mathbf{Y} \right\},
\end{equation}
where $\mathbf{X}$ and $\mathbf{Y}$ are a random vector in $\mathcal{X}$  and a random variable in $\mathcal{Y}$, respectively. 
Most adversarial training algorithms learn $\bm{\theta}$
by minimizing an empirical version of the above robust population risk. 
In turn, most empirical versions of (\ref{eqn1})
require to generate an \textit{adversarial example}
which is a surrogate version of
$$\bm{x}^{\text{adv}}:= \underset{\bm{x}' \in \mathcal{B}_p(\bm{x}, \varepsilon) }{\operatorname{argmax}\;\;\;} \mathbbm{1}\left\{F_{\theta}(\bm{x'}) \neq y \right\}.$$
Any method of generating an adversarial example is called an \textit{adversarial attack.}

\subsection{Algorithms for Generating Adversarial Examples}
Existing adversarial attacks can be categorized into either the white-box attack \citep{goodfellow2015explaining, madry2018towards, carlini2017evaluating, croce2020minimally} or the black-box attack \citep{papernot2016transferability, papernot2017practical, chen2017zoo, ilyas2018blackbox, papernot2016science}. For the white-box attack, the model structure and parameters are known to adversaries who use this information for generating adversarial examples, while outputs for given inputs are only available to adversaries for the black-box attack.
The most popular method for the white-box attack is PGD (Projected Gradient Descent) with infinite norm \citep{madry2018towards}. Let $\ell(\bm{x}'|\theta,\bm{x},y)$ be a surrogate loss of $\mathbbm{1}\left\{F_{\theta}(\bm{x'}) \neq y \right\}$ for given $\bm{\theta},\bm{x},y$.
%and $\bm{x}' \in {\mathcal{B}_{p}(\bm{x}, \varepsilon)}.$
PGD finds the adversarial example 
by applying the gradient ascent algorithm to $\ell$  to update $\bm{x}'$ and projecting it to $\mathcal{B}_{\infty}(\bm{x}, \varepsilon).$ That is, the update rule of PGD is 
\begin{equation}
\footnotesize
\label{pgd}
    \bm{x}^{(m+1)}=\bm{\Pi}_{\mathcal{B}_{\infty}(\bm{x}, \varepsilon) }\left(\bm{x}^{(m)} + \eta \operatorname{sgn}\left(\nabla_{\bm{x}^{(m)}} \ell(\bm{x}^{(m)}|\bm{\theta},\bm{x},y)\right)\right),
\end{equation}
where $\eta>0$ is the step size, 
$\bm{\Pi}_{\mathcal{B}_{\infty}(\bm{x}, \varepsilon)}(\cdot)$ is the projection operator to $\mathcal{B}_{\infty}(\bm{x}, \varepsilon)$ and $\bm{x}^{(0)}=\bm{x}$. We define $\bm{x}^{\text{pgd}}$ as $\bm{x}^{\text{pgd}} := \lim \limits_{m \rightarrow \infty} \bm{x}^{(m)}$ and denote the proxy by $\widehat{\bm{x}}^{\text{pgd}}=\bm{x}^{(M)}$ with finite step $M$.
For the surrogate loss $\ell$, the cross entropy \citep{madry2018towards}
or the KL divergence \citep{zhang2019theoretically} is used.

%\begin{equation*}
%    \bm{x}^{\text{adv}}_{\eta}= \underset{\bm{x}' \in {\mathcal{B}_{p}(\bm{x}, \varepsilon)}}{\argmax} %\eta(\bm{x}'|\theta,\bm{x},y)
%\end{equation*}

For the black-box attack, an adversary generates a dataset $\left\{\bm{x}_i, \tilde{y}_i \right\}_{i=1}^n$ where $\tilde{y}_i$ is an output of a given input $\bm{x}_i$.
Then, the adversary trains a substitute prediction model based on this data set, and generates adversarial examples from the substitute prediction model by PGD \citep{papernot2017practical}.

\subsection{Review of Adversarial Training Algorithms}
We review some of the adversarial training algorithms which, we think, are related to our proposed algorithm. Typically, adversarial training algorithms consist of the maximization and minimization steps. In the maximization step, we generate adversarial examples for given $\bm{\theta}$, and in the minimization step, we fix the adversarial examples and update $\bm{\theta}$. In the followings, we denote $\bm{\widehat{x}}^{\text{pgd}}_{i}$ as the adversarial example corresponding to
$(\bm{x}_i, y_i)$ generated by PGD.
%For notational simplicity, we drop $\bm{\theta}$ in the adversarial examples.

\subsubsection{Algorithms minimizing the robust risk directly}

\paragraph{PGD-AT} \citet{madry2018towards} proposes PGD-AT which 
updates $\bm{\theta}$ by minimizing
\begin{equation*}
    \sum\limits_{i=1}^n \ell_{\text{ce}}(f_{\bm{\theta}}(\bm{\widehat{x}}^{\text{pgd}}_{i}), y_i), 
\end{equation*}
where $\ell_{\text{ce}}$ is the cross-entropy loss.

\paragraph{GAIR-AT}
Geometry Aware Instance Reweighted Adversarial Training (GAIR-AT) \citep{zhang2021geometry} 
is a modification of PGD-AT, where
the weighted robust risk is minimized and
more weights are given to samples closer to the
decision boundary. To be more specific, the weighted empirical risk of GAIR-AT is given as
\begin{equation*}
    \sum\limits_{i=1}^n w_{\theta}(\bm{x}_i, y_i) \ell_{\text{ce}}(f_{\bm{\theta}}(\widehat{\bm{x}}^{\text{pgd}}_i), y_i),
\end{equation*}
where $\kappa_{\theta}(\bm{x}_i, y_i) = \min \left( \min (\{t : F_{\bm{\theta}}(\bm{x}_i^{(t)}) \neq y_i \}), T \right)$ for a prespecified maximum iteration $T$ and $w_{\theta}(\bm{x}_i, y_i)= (1+\operatorname{tanh}(5(1-2\kappa_{\theta}(\bm{x}_i, y_i)/T))) / 2$.

There are other similar modifications of PGA-AT 
including Max-Margin Adversarial (MMA) Training  \citep{ding2020mma} and Friendly Adversarial Training (FAT) \citep{zhang2020attacks}.

%\paragraph{MMA}
%\citet{ding2020mma} suggests to generate adversarial examples only for correctly classified samples with data-adaptive neighborhood size $\varepsilon$ in PGD.
%Max-Margin Adversarial (MMA) Training of \citet{ding2020mma} minimizes
%\begin{align*}
%    \sum\limits_{i=1}^n &\Big\{ \ell_{\text{ce}}(f_{\bm{\theta}}(\bm{x}_i), y_i) \mathbbm{1}(y_i \neq F_{\bm{\theta}}(\bm{x}_i))
%    + \ell_{\text{ce}}(f_{\bm{\theta}}(\widehat{\bm{x}}^{\text{pgd}}_{i,\varepsilon_i}), y_i) \mathbbm{1}(y_i = F_{\bm{\theta}}(\bm{x}_i) \Big\},
%\end{align*}
%where $\widehat{\bm{x}}^{\text{pgd}}_{i,\varepsilon_i}$ is an adversarial example by PGD with data-adaptively selected $\varepsilon_i$.

\subsubsection{Algorithms minimizing a regularized empirical risk}
\label{al_reg}
Robust risk, natural risk and boundary risk are defined by 
\begin{align*}
    \mathcal{R}_{\text{rob}}(\bm{\theta}) &= \mathbb{E}_{(\mathbf{X}, Y)}\mathbbm{1}\left\{\exists \mathbf{X}'\in \mathcal{B}_p(\mathbf{X}, \varepsilon) : F_{\bm{\theta}}(\mathbf{X}') \neq Y \right\},  \\
    \mathcal{R}_{\text{nat}}(\bm{\theta}) &= \mathbb{E}_{(\mathbf{X}, Y)}\mathbbm{1}\left\{ F_{\bm{\theta}}(\mathbf{X}) \neq Y \right\},\\
    \mathcal{R}_{\text{bdy}}(\bm{\theta}) &= \mathbb{E}_{(\mathbf{X}, Y)}\mathbbm{1}\{\exists \mathbf{X}' \in \mathcal{B}_p(\mathbf{X}, \varepsilon) \\ 
    & \qquad\qquad\; : F_{\bm{\theta}}(\mathbf{X})\neq F_{\bm{\theta}}(\mathbf{X}'), F_{\bm{\theta}}(\mathbf{X}) = Y   \}.
\end{align*}
\citet{zhang2019theoretically} shows 
\begin{equation*}
    \mathcal{R}_{\text{rob}}(\bm{\theta}) = \mathcal{R}_{\text{nat}}(\bm{\theta}) + \mathcal{R}_{\text{bdy}}(\bm{\theta}).
\end{equation*}
By treating $\mathcal{R}_{\text{bdy}}(\bm{\theta})$ as the regularization term,
various regularized risks for adversarial training have been proposed.

\paragraph{TRADES}
\citet{zhang2019theoretically} proposes the following regularized empirical risk which is a surrogate version of the upper bound of the robust risk:
\begin{equation*}
\label{trades}
    \sum\limits_{i=1}^n \left\{ \ell_{\text{ce}}(f_{\theta}(\bm{x}_i), y_i) + \lambda \cdot \operatorname{KL} (\mathbf{p}_{\theta }(\cdot|\bm{x}_i)\lVert \mathbf{p}_{\theta }(\cdot|\widehat{\bm{x}}^{\text{pgd}}_i))\right\},
\end{equation*}

\paragraph{HAT}
Helper based training \citep{rade2022reducing} is a variation of TRADES where an additional regularization term based on helper examples
is added to the regularized risk. The role of helper examples is to restrain the decision boundary from
having excessive margins. HAT minimizes the following regularized empirical risk:
\begin{align}
   \sum_{i=1}^{n} \Bigg\{ & \ell_{\text{ce}}\left(f_{\bm{\theta}}\left(\bm{x}_{i}\right), y_{i}\right)+ \lambda \cdot \operatorname{KL}\left(\mathbf{p}_{\theta}\left(\cdot | \bm{x}_{i}\right) \| \mathbf{p}_{\theta} (\cdot | \widehat{\bm{x}}^{\text{pgd}}_{i})\right) \nonumber \\
   & +\gamma \ell_{\text{ce}}\left(f_{\bm{\theta}}(\bm{x}_{i}^{\text{helper}}), F_{\bm{\theta}_{\text{pre}}}(\widehat{\bm{x}}^{\text{pgd}}_{i})
    \right) \Bigg\} \label{hat},
\end{align}
where $\bm{\theta}_{\text{pre}}$ is the parameter of a pre-trained model only with clean examples, 
$\bm{x}_{i}^{\text{helper}} = \bm{x}_i +2 (\widehat{\bm{x}}^{\text{pgd}}_{i} - \bm{x}_i)$.

\paragraph{MART}
Misclassification Aware adveRsarial Training (MART) \citep{wang2020improving} minimizes
\begin{align}
    \sum\limits_{i=1}^n \bigg\{ & \ell_{\text{margin}}(f_{\bm{\theta}}(\widehat{\bm{x}}^{\text{pgd}}_i), y_i) \nonumber \\
    & + \lambda \cdot  \operatorname{KL}(\mathbf{p}_{\bm{\theta}}(\cdot|\bm{x}_i) \lVert \mathbf{p}_{\bm{\theta}}(\cdot|\bm{\widehat{x}}^{\text{pgd}}_{i})) (1- p_{\bm{\theta}}(y_i|\bm{x}_i)) \bigg\}, \label{mart}
\end{align}
where $\ell_{\text{margin}}(f_{\bm{\theta}}(\widehat{\bm{x}}^{\text{pgd}}_i), y_i) = -\log p_{\bm{\theta}}(y_i|\widehat{\bm{x}}^{\text{pgd}}_i)-\log(1-\underset{k \neq y_i}{\max}\; p_{\bm{\theta}}(k|\widehat{\bm{x}}^{\text{pgd}}_i))$.
This objective function can be regarded as the regularized robust risk and thus
MART can be considered as a hybrid algorithm of PGD-AT and TRADES.
%--------------------------------------------------------------------------------------------------------------------------
\section{Anti-Robust Weighted Regularization (ARoW)}
In this section, we develop a new adversarial training algorithm called Anti-Robust Weighted Regularization (ARoW), which is an algorithm
minimizing a regularized risk. We propose
a new regularized term which applies
more regularization to data vulnerable to
adversarial attacks than other existing algorithms such as TRADES and HAT do.
Our new regularized term is motivated by the upper bound of the robust risk derived in the following section.

%The two upper bounds are the upper bound and thus the corresponding adversarial training algorithm is expected to perform better than TRADES.
%First, we derive two upper bounds of the robust (population) risk which can be interpreted as data-adaptive regularized risks. Then, we propose to minimize the empirical counterparts of the two upper bounds for adversarial training. 

\subsection{Upper Bound of the Robust Risk}
\label{upperbound:arow}

In this subsection, we provide an upper bound of the robust risk for multi-classification problem which is stated in the following theorem. The proof is deferred to Appendix \ref{appA}.
%Derivation of upper bounds for binary-classification problem, which is covered on TRADES \citep{zhang2019theoretically}, is provided in Appendix \ref{appA_3}.

\begin{restatable}{theorem}{multi}
\label{theorem1}
    For a given score function $f_{\bm{\theta}}$, let $z(\cdot)$ be an any measurable mapping from $\mathcal{X}$ to $\mathcal{X}$ satisfying
\begin{equation*}
    z(\bm{x}) \in \underset{\bm{x}' \in \mathcal{B}_{p}(\bm{x}, \varepsilon)}{\operatorname{argmax}} \mathbbm{1} \left( F_{\bm{\theta}}(\bm{x}) \neq F_{\bm{\theta}}(\bm{x}')\right).
\end{equation*}
for every $\bm{x} \in \mathcal{X}$. Then, we have
\begin{align} 
    & \mathcal{R}_{\text{rob}}(\bm{\theta}) \leq \mathbb{E}_{(\mathbf{X},Y)} \mathbbm{1}(Y \neq F_{\bm{\theta}}(\mathbf{X})) \nonumber \\
    & + \mathbb{E}_{(\mathbf{X}, Y)}{\mathbbm{1}(F_{\bm{\theta}}(\mathbf{X}) \neq F_{\bm{\theta}}(z(\mathbf{X})))} \mathbbm{1} \left\{ p_{\bm{\theta}}(Y | z(\mathbf{X})) < 1/2 \right\}
    \label{arow:ub}
\end{align}
\end{restatable}

The upper bound (\ref{arow:ub}) consists of the two terms : the first term is the natural risk itself and the second term is an upper bound of the boundary risk.
This upper bound is motivated by the upper bound derived in TRADES \citep{zhang2019theoretically}. For binary classification problems, \cite{zhang2019theoretically} shows that
\begin{equation} 
    \mathcal{R}_{\text{rob}}(\bm{\theta}) \leq \mathbb{E}_{(\mathbf{X},Y)} \phi(Y f_{\bm{\theta}}(\mathbf{X})) +
    \mathbb{E}_{\mathbf{X}}{\phi(f_{\bm{\theta}}(\mathbf{X}) f_{\bm{\theta}}(z(\mathbf{X})))}, \label{arow:ub_trades}
\end{equation}
where 
$$
z(\bm{x}) \in \underset{\bm{x}' \in \mathcal{B}_{p}(\bm{x}, \varepsilon)}{\operatorname{argmax}} \phi \left( f_{\bm{\theta}}(\bm{x}) f_{\bm{\theta}}(\bm{x}')\right)
$$
and $\phi(\cdot)$ is an upper bound of
$\mathbbm{1}(\cdot <0).$
Our upper bound (\ref{arow:ub}) is a modification of the upper bound (\ref{arow:ub_trades}) for multiclass problems
where $\phi(\cdot)$ and $f_{\bm{\theta}}$ in (\ref{arow:ub_trades})
are replaced by $\mathbbm{1}(\cdot <0)$
 and $F_{\bm{\theta}},$ respectively.  
A key difference, however, between
(\ref{arow:ub}) and (\ref{arow:ub_trades})
is the term $\mathbbm{1} \left\{ p_{\bm{\theta}}(Y | z(\mathbf{X})) < 1/2 \right\}$ at the last part of (\ref{arow:ub})
that is not in (\ref{arow:ub_trades}). 
%That is, our upper bound (\ref{arow:ub})
%is tighter than the upper bound (\ref{arow:ub_trades}) even after we replace
%$\phi(\cdot)$ by $\mathbbm{1}(\cdot<0)$ in (\ref{arow:ub_trades}). 

It is interesting to see that 
the upper bound in Theorem \ref{theorem1} becomes equal to the robust risk for binary classification problems.
That is, the upper bound (\ref{arow:ub}) is an another formulation of the robust risk. However,
this rephrased formula of the robust risk is useful since it provides a new learning algorithm
when the indicator functions are replaced by their surrogates as we do.

\subsection{Algorithm}

\begin{algorithm}[H]
    \small
    \caption{\footnotesize{Anti-Robust Weighted (ARoW) Regularization}}
    \label{alg:arow}
    \textbf{Input} : network $f_{\bm{\theta}}$, training dataset $\mathcal{D}=\{(\bm{x}_i, y_i) \}_{i=1}^n$, learning rate $\eta$,
    perturbation budget $\varepsilon$, number of PGD steps $M$,
    hyperparameters ($\lambda$, $\alpha$) of (\ref{arow:surrogate}), number of epochs $T$, number of batch $B$, batch size $K$. \\
    \textbf{Output} : adversarially robust network $f_{\bm{\theta}}$
	\begin{algorithmic}[1]
        \FOR{$ t = 1 , \cdots, T$}
            \FOR{$ b = 1 , \cdots, B$}
                \FOR{$ k = 1 , \cdots, K$}
                    \STATE Generate $\widehat{\bm{x}}^{\text{pgd}}_{b, k}$ using PGD$^{(M)}$ in (\ref{pgd}) \\
                    
                \ENDFOR
                \STATE 
                \begin{align*}
                \bm{\theta} \leftarrow 
                \bm{\theta}-\eta\frac{1}{K} \nabla_{\bm{\theta}} & \mathcal{R}_{\text{ARoW}}({\bm{\theta}} \; ; \{(\bm{x}_k, y_k)\}^K_{k=1}, \lambda, \alpha) %\text{\;where\;} \mathcal{R}_{\text{ARoW}} \text{\;is\;} (\ref{arow:surrogate}).
                \end{align*}
            \ENDFOR
        \ENDFOR
        \STATE \textbf{Return} $f_{\bm{\theta}}$
	\end{algorithmic}
\end{algorithm}

By replacing the indicator functions in Theorem
\ref{theorem1} by their smooth proxies,
we propose a new regularized risk and develop
the corresponding adversarial learning algorithm called the Anti-Robust Weighted Regularization (ARoW) algorithm. The four indicator functions in (\ref{arow:ub}) are
replaced by
\begin{itemize}
\item the adversarial example $z(\bm{x})$
is replaced by $\widehat{\bm{x}}^{\text{pgd}}$
obtained by the PGD algorithm with the KL divergence or cross entropy;
\item the term $\mathbbm{1}(Y \neq F_{\bm{\theta}}(\mathbf{X}))$
is replaced by the label smooth cross-entropy \citep{muller2019when}
 $\ell^{\text{LS}}(f_{\bm{\theta}}(\bm{x}), y) = - {\bm{y}_{\alpha}^{\text{LS}}}^{\top} \log \mathbf{p}_{\theta}(\cdot|\bm{x})$ 
 for a given $\alpha>0,$ where $\bm{y}_{\alpha}^{\text{LS}} = (1-\alpha)\mathbf{u}_y + \frac{\alpha}{C}\mathbf{1}_C$, $\mathbf{u}_y \in \mathbb{R}^{C}$ is the one-hot vector whose the $y$-th entry is 1 and $\mathbf{1}_C \in \mathbb{R}^{C}$ is the vector whose entries are all 1;
\item the term $\mathbbm{1}( F_{\bm{\theta}}(\mathbf{X})
\neq F_{\bm{\theta}}(z(\mathbf{X})))$
is replaced by
$ \lambda \cdot \operatorname{KL}(\mathbf{p}_{\bm{\theta}}(\cdot|\mathbf{X}) || \mathbf{p}_{\bm{\theta}}(\cdot|\widehat{\mathbf{X}}^{\text{pgd}}))$
for $\lambda>0;$

\item the term 
$\mathbbm{1} \left\{ p_{\bm{\theta}}(Y | z(\mathbf{X})) < 1/2 \right\}$ is replaced
by its convex upper bound $2 (1- p_{\bm{\theta}}(Y | \widehat{\mathbf{X}}^{\text{pgd}} ));$
 \end{itemize}
 to have the following regularized risk for ARoW, which is a smooth surrogate of the upper bound (\ref{arow:ub}),
 \begin{align}
    &\mathcal{R}_{\text{ARoW}}(\bm{\theta} ; \left\{(\bm{x}_i, y_i) \right\}_{i=1}^n, \lambda) \nonumber \\
    &:= \sum\limits_{i=1}^n \bigg\{ \ell^{\text{LS}}(f_{\bm{\theta}}(\bm{x}_i), y_i) 
    \nonumber \\ 
    & + 2 \lambda \cdot \operatorname{KL}(\mathbf{p}_{\bm{\theta}}(\cdot|\bm{x}_i) || \mathbf{p}_{\bm{\theta}}(\cdot|\widehat{\bm{x}}^{\text{pgd}}_i)) \cdot (1 - p_{\bm{\theta}}(y_i|\widehat{\bm{x}}^{\text{pgd}}_i)) \bigg\}
    \label{arow:surrogate}.
\end{align}
Here, we introduce the regularization parameter $\lambda>0$ to control the robustness of a trained prediction model to adversarial attacks. That is, the regularized risk (\ref{arow:surrogate}) can be considered as a smooth surrogate
of the regularized robust risk
of $\mathcal{R}_{\text{nat}}(\bm{\theta})
+ \lambda \mathcal{R}_{\text{bdy}}(\bm{\theta}).$

We use the label smoothing cross-entropy as a surrogate for $\mathbbm{1}(Y \neq F_{\bm{\theta}}(\mathbf{X}))$
instead of the standard cross-entropy
to estimate the conditional class probabilities $\mathbf{p}_\theta(\cdot|\bm{x})$ more accurately \citep{muller2019when}.
The accurate estimation of $\mathbf{p}_{\bm{\theta}}(\cdot|\bm{x})$ is important since it is used in the regularization term of ARoW. It is well known that DNNs trained by minimizing the
cross-entropy are poorly calibrated \citep{guo2017on}, and so we use the label smoothing cross-entropy technique.
%We set $\alpha=0.2$ in our numerical studies for simplicity even if it can be tuned optimally.

The ARoW algorithm, which learns $\bm{\theta}$
by minimizing $\mathcal{R}_{\text{ARoW}}(\bm{\theta} ; \left\{(\bm{x}_i, y_i) \right\}_{i=1}^n, \lambda),$ is summarized in Algorithm \ref{alg:arow}.

%-----------------------------------------------
\begin{table*}[ht]
    \footnotesize
    \caption{\textbf{Comparison of ARoW and Other Competitors.} We conduct the experiment three times with different seeds and present the averages of the accuracies with the standard errors in the brackets.}
    \centering
    \begin{tabular}{c|ccc|ccc}
    \toprule
     \multirow{2}{*}{\textbf{Method}} &
        \multicolumn{3}{c|}{CIFAR10 (WRN-34-10)} &
        \multicolumn{3}{c}{CIFAR100 (WRN-34-10)} \\
    \cline{2-7}    
     & \textbf{Stand}  & $\textbf{PGD}^{20}$ & \textbf{AA} & \textbf{Stand}  & $\textbf{PGD}^{20}$ & \textbf{AA} \\
     \hline
    PGD-AT                        & 87.02(0.20) & 57.50(0.12) & 53.98(0.14) & 62.20(0.11) & 32.27(0.05) & 28.66(0.05) \\
    GAIR-AT                         & 85.44(0.10) & \textcolor{red}{67.27}(0.07) & \textcolor{red}{46.41}(0.07) & 62.25(0.12) & \textcolor{red}{30.55}(0.04) & \textcolor{red}{24.19}(0.16) \\
    TRADES                      & 85.86(0.09) & 56.79(0.08) & 54.31(0.08) & 
    62.23(0.07) & 33.45(0.22) & 29.07(0.25) \\
    HAT                         & 86.98(0.10) & 56.81(0.17) & 54.63(0.07) &  60.42(0.03) & 33.75(0.08) & 29.42(0.02) \\
    MART                        & 83.17(0.18) & 57.84(0.13) & 51.84(0.09) & 59.76(0.13) & 33.37(0.11) & 29.68(0.08) \\
    ARoW                     & \textbf{87.65}(0.02) & \textbf{58.38}(0.09) & \textbf{55.15}(0.14) &   \textbf{62.38}(0.07) & \textbf{34.74}(0.11) & \textbf{30.42}(0.10) \\
    %\textbf{84.53}(0.13) & \textbf{53.70}(0.16) & \textbf{49.98}(0.18) \\
    \hline
    \hline
     \multirow{2}{*}{\textbf{Method}} &
        \multicolumn{3}{c|}{SVHN (ResNet-18)} &
        \multicolumn{3}{c}{FMNIST (ResNet-18)} \\
    \cline{2-7}
     & \textbf{Stand}  & $\textbf{PGD}^{20}$ & \textbf{AA} & \textbf{Stand}  & $\textbf{PGD}^{20}$ & \textbf{AA} \\
     \hline
    PGD-AT                       & 92.75(0.04) & 59.05(0.46) & 47.66(0.52) &  92.25(0.06) & 87.43(0.03) & 87.19(0.03) \\
    GAIR-AT                         & 91.95(0.40) & \textcolor{red}{70.29}(0.18) & \textcolor{red}{38.26}(0.48) &  90.96(0.10) & 87.25(0.01) & 85.00(0.12) \\ 
    TRADES                      & 91.62(0.49) & 58.75(0.19) & 51.06(0.93) &  91.92(0.04) & 88.33(0.03) & 88.19(0.04) \\
    HAT                         & 91.72(0.12) & 58.66(0.06) & 51.67(0.12) &  92.10(0.11) & 88.09(0.16) & 87.93(0.13) \\
    MART                       & 91.64(0.41) & 60.57(0.27) & 49.95(0.42) &  92.14(0.05) & 88.10(0.10) & 87.88(0.14) \\
    ARoW                     & \textbf{92.79}(0.24) & \textbf{61.14}(0.74) & \textbf{51.93}(0.33) &  \textbf{92.26}(0.05) & \textbf{88.73}(0.03) & \textbf{88.54}(0.04)\\
    %\textbf{84.53}(0.13) & \textbf{53.70}(0.16) & \textbf{49.98}(0.18) \\
    \bottomrule
  \end{tabular}
  \label{table:compare}
\end{table*}
%-----------------------------------------------
%-----------------------------------------------
\begin{figure*}[ht]
\begin{center}
\begin{minipage}[c]{0.45\linewidth}
\includegraphics[width=\linewidth]{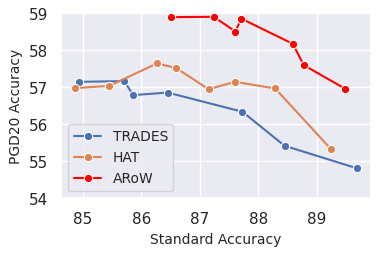}
\end{minipage}
\hfill
\begin{minipage}[c]{0.45\linewidth}
\includegraphics[width=\linewidth]{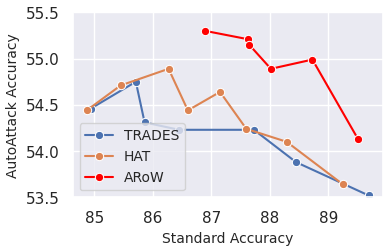}
\end{minipage}
\caption{\textbf{Comparison of ARoW, TRADES and HAT with varying $\lambda$}.
The $x$-axis and $y$-axis are the standard and robust accuracies, respectively. The robust accuracies in the left panel are against PGD$^{20}$ while the robust accuracies in the right panel are against AutoAttack. 
We exclude the results of MART from the figures because its roboust against autoattack and standard accuracies are too low.}
\label{fig1}
\end{center}
\end{figure*}
%-----------------------------------------------

\paragraph{Comparison to TRADES}

A key difference of the regularized risks
of ARoW and TRADES is that TRADES does not have
the term  $(1- p_{\bm{\theta}}(y_i|\widehat{\bm{x}}^{\text{pgd}}_i))$
at the last part of (\ref{arow:surrogate}).
That is, ARoW puts more regularization to
samples which are vulnerable to
adversarial attacks (i.e. $p_{\bm{\theta}}(y_i|\widehat{\bm{x}}^{\text{pgd}}_i)$ is small). Note that this term is motivated by the tighter upper bound of the robust risk (\ref{arow:ub}) and thus is expected to lead better results. Numerical studies confirm that it really works.

\paragraph{Comparison to MART}
The objective function in MART (\ref{mart}) is similar with the objective function of ARoW. 
But, there are two main differences. 
First, the supervised loss term of ARoW is the label smoothing loss with clean examples, whereas MART uses the margin cross entropy loss with adversarial examples.
Second, the regularization term in MART is proportional to $(1-p_{\bm{\theta}}(y|\bm{x}))$  while that in ARoW is proportional to $(1-p_{\bm{\theta}}(y|\widehat{\bm{x}}^{\text{pgd}})).$
Even though these two terms look similar, their roles are quite different.
In Appendix \ref{thm:cow}, we derive an upper bound of the robust risk which suggests
$p_{\bm{\theta}}(y|\bm{x})$ as the regularization term that is completely opposite to
that for MART. Numerical studies in Appendix \ref{compare-cow-mart} show that 
the corresponding algorithm, called Confidence Weighted regularization (CoW), outperforms MART with large margins, which indicates that the regularization term is MART
would be suboptimal. Note that ARoW is better than CoW even if the differences are not large.

%-----------------------------------------------
\begin{table*}[ht]
    \caption{\textbf{Comparison of ARoW to other adversarial algorithms with extra data on CIFAR10.}}
    \centering
    \begin{tabular}{c|c|ccccc}
    \hline
    \textbf{Model} & \textbf{Extra data} & \textbf{Method}  & \textbf{Stand}  & $\textbf{PGD}^{20}$ & \textbf{AutoAttack} \\
    \hline
    \hline
    \multirow{4}{*}{WRN-28-10} & \multirow{4}{*}{80M-TI(500K)}
         & \citet{carmon2019unlabeled}  & 89.69 & 62.95 & 59.58 & \\
    &    & \citet{rebuffi2021data}      & 90.47 & 63.06 & 60.57 & \\
    &    & HAT                      & 91.50 & 63.42 & \textbf{60.96} & \\
    &    & ARoW                     & \textbf{91.57} & \textbf{64.64} & 60.91 & \\
    \cline{2-7}
    %& \multirow{2}{*}{DDPM(1M)}
    %  & HAT                 & 87.50 & 61.11 & 58.67  \\
    %& & ARoW                & \textbf{88.34} & \textbf{61.98} & \textbf{59.01}  &\\
    \hline
    \multirow{8}{*}{ResNet-18} & \multirow{4}{*}{80M-TI(500K)}
    & \citet{carmon2019unlabeled}  &  87.07 & 56.86 & 53.16 &   \\
    & & \citet{rebuffi2021data}    & 87.67 & 59.20 & 56.24  \\
    & & HAT                     & 88.98 & 59.29 & 56.40   \\
    & & ARoW                  &  \textbf{89.04} & \textbf{60.38} & \textbf{56.54} &   \\
    \cline{2-7}
    & \multirow{4}{*}{DDPM(1M)}
    & \citet{carmon2019unlabeled} & 82.61 & 56.16 & 52.82 & \\
    & & \citet{rebuffi2021data}  & 83.46 & 56.89 & 54.22 \\
    & & HAT                   & 86.09 & 58.61 & 55.44 \\
    & & ARoW                & \textbf{86.72} & \textbf{59.50} & \textbf{55.57} &\\
    \hline
    \end{tabular}
    \label{table-extra}
    \vskip -0.1in
\end{table*}
%-----------------------------------------------
\section{Experiments}
\label{sec4}

In this section, we investigate ARoW algorithm in view of robustness and generalization by analyzing the four benchmark data sets - CIFAR10, CIFAR100 \citep{krizhevsky09learningmultiple} , F-MINST \citep{xiao2017fahsion} and SVHN dataset \citep{netzer2011svhn}.
In particular, we show that ARoW is superior to existing algorithms including TRADES \citep{zhang2019theoretically}, HAT \citep{rade2022reducing} and MART \citep{wang2020improving}
as well as PGD-AT \citep{madry2018towards} and GAIR-AT \citep{zhang2021geometry}
to achieve state-of-art performances.
WideResNet-34-10 (WRN-34-10) \citep{zagoruyko2016wide} 
and ResNet-18 \citep{he2016deep} are used for CIFAR10 and CIFAR100 while 
ResNet-18 \citep{he2016deep} is used for F-MNIST and SVHN.
We apply SWA for mitigating robust overfitting \cite{chen2021robust} on CIFAR10 and CIFAR100. The effect of SWA are described in Section \ref{ablation}.
Experimental details are presented in Appendix \ref{appB}.
The code is available at \href{https://github.com/dyoony/ARoW}{https://github.com/dyoony/ARoW}.

%-----------------------------------------------
\subsection{Comparison of ARoW to Other Competitors}
\label{sec4_1}
We compare ARoW to other competitors
TRADES \citep{zhang2019theoretically}, HAT \citep{rade2022reducing}, MART \citep{wang2020improving} explained in Section \ref{al_reg},
PGD-AT \citep{madry2018towards} and GAIR-AT \citep{zhang2021geometry} which are the algorithms  minimizing the robust risk directly.

Table \ref{table:compare} shows that 
ARoW outperforms the other competitors for various data sets and architectures
in terms of the standard accuracy and the robust accuracy against to AutoAttack \citep{croce2020reliable}.
GAIR-AT is, however, better for PGD$^{20}$ attack than ARoW.
This would be due to the gradient masking \citep{papernot2016science, papernot2017practical} as described in Appendix \ref{gradient-masking}.
The selected values of the hyper-parameters for the other algorithms are listed in Appendix B.2.

To investigate whether ARoW dominates its competitors uniformly with respect to the regularization 
parameter $\lambda,$ we compare the trade-off between the standard and robust accuracies of ARoW
and other regularization algorithms when $\lambda$ varies.
Figure \ref{fig1} draws the plots of the standard accuracies 
in the $x$-axis and the robust accuracies in the $y$-axis obtained by the corresponding algorithms with
various values of $\lambda.$ For this experiment, we use CIFAR10 and WideResNet-34-10 (WRN-34-10) architecture.

The trade-off between the standard and robust accuracies is well observed (i.e. a larger regularization parameter $\lambda$ yields lower standard accuracy but higher robust accuracy).
Moreover, we can clearly see that ARoW uniformly dominates TRADES and HAT
(and MART) regardless of the choice of the regularization parameter and the methods for adversarial attack. 
Additional results for the trade-off are provided in Appendix \ref{ablation:lambda}. 

\begin{table*}[ht]
    \caption{\textbf{Comparison of MART and ARoW}. We compare the robustness of MART \citep{wang2020improving} and ARoW against the four attacks used  in AutoAttack on CIFAR10. The results are based on WRN-34-10. We set $\lambda=3$ and ARoW, respectively.}
    \centering
    \begin{tabular}{c|c|cccc}
    \hline
    \textbf{Method} & \textbf{Standard}  & \textbf{APGD} & \textbf{APGD-DLR}  & \textbf{FAB} & \textbf{SQUARE} \\
    \hline
    \multirow{1}{*}{\text{MART}}
    & 83.17 & 56.30 & 51.87 & 51.28  & 58.59 \\
    %\multirow{1}{*}{\text{CoW}}
    %& 88.53 & 57.22 & 56.37  & 54.59 & 56.37 &   \\
    \multirow{1}{*}{\text{ARoW}}
    & \textbf{87.65} & \textbf{56.37} & \textbf{55.17} & \textbf{56.69} & \textbf{63.50} \\
    \hline
    \end{tabular}
    \label{compare-mart}
\end{table*}

%-----------------------------------------------
 \begin{table*}[ht]
    \caption{\textbf{Role of the new regularization term in ARoW.} \# $\textbf{Rob}_{\text{TRADES}}$ and \# $\textbf{Rob}_{\text{ARoW}}$ represent the number of samples which are robust to TRADES and ARoW, respectively. \textbf{Diff.} and \textbf{Rate of Impro.} denote (\# $\textbf{Rob}_{\text{ARoW}}$ - \# $\textbf{Rob}_{\text{TRADES}}$) and \textbf{Diff.} / \# $\textbf{Rob}_{\text{TRADES}}$). The $\text{PGD}^{10}$ is used for evaluating the robustness.}
    \centering
    \begin{tabular}{r|cccc}
    \hline
    Sample's Robustness & \# $ \textbf{Rob}_{\text{TRADES}}$ & \# $\textbf{Rob}_{\text{ARoW}}$ & \textbf{Diff.} & \textbf{Rate of Impro. (\%) } \\
    \hline
    \hline
    Least Robust  & 317  & 357  & 40 & 12.62 \\
    Less Robust   & 945  & 1008 & 63 & 6.67 \\
    Robust             & 969  & 1027 & 58 & 5.99 \\
    Highly Robust      & 3524 & 3529 & 5 & 0.142\\
    \hline
    \end{tabular}
    \label{effect-regularization-trades}
    \vskip -0.1in
\end{table*}
%-----------------------------------------------
 \begin{table*}
    \caption{\textbf{Comparing of ARoW to MART on sample's robustness.} \# $\textbf{Rob}_{\text{MART}}$ and \# $\textbf{Rob}_{\text{ARoW}}$ represent the number of samples which are robust to MART and ARoW, respectively. \textbf{Diff.} and \textbf{Rate of Impro.} denote (\# $\textbf{Rob}_{\text{ARoW}}$ - \# $\textbf{Rob}_{\text{MART}}$) and \textbf{Diff.} / \# $\textbf{Rob}_{\text{MART}}$). The autoattack is used for evaluating the robustness because of gradient masking.}
    \centering
    \begin{tabular}{r|cccc}
    \hline
    Sample's Robustness & \# $ \textbf{Rob}_{\text{MART}}$ & \# $\textbf{Rob}_{\text{ARoW}}$ & \textbf{Diff.} & \textbf{Rate of Impro. (\%) } \\
    \hline
    \hline
    Least Robust  & 150  & 148 & -2 & -1.3 \\
    Less Robust   & 729  & 865 & \textbf{136} & 18.65 \\
    Robust             & 962  & 984 & 22 & 2.29 \\
    Highly Robust      & 3515 & 3530 & 15 & 0.04\\
    \hline
    \end{tabular}
    \label{effect-regularization-mart}
    \vskip -0.1in
\end{table*}
%-----------------------------------------------
\paragraph{Experimental comparison to MART}
We observe that MART has relatively high robust accuracies against PGD-based attacks than other attacks. Table \ref{compare-mart} shows the robust accuracies against four attacks included in AutoAttack \cite{croce2020reliable}.
Table \ref{compare-mart} shows that MART has good performance for APGD, but not for APGD-DLR, FAB and SQUARE. 
This result indicates that the gradient masking occurs for MART. That is,
PGD does not find good adversarial examples, but the other attacks easily find adversarial examples.
See Appendix \ref{gradient-masking} for details about gradient masking.
%Details about gradient masking is provided in Appendix \ref{gradient-masking}.

%-----------------------------------------------
\subsection{Analysis with extra data}
\label{sec4_3}
For improving performance on CIFAR10, \cite{carmon2019unlabeled} and \cite{rebuffi2021data} use extra unlabeled data sets with TRADES. \cite{carmon2019unlabeled} uses an additional subset of 500K extracted from 80 Million Tiny Images (80M-TI) and \cite{rebuffi2021data} uses a data set of 1M synthetic samples generated by a denoising diffusion probabilistic model (DDPM) \citep{ho2020denoising} along with the SiLU activation function and Exponential Moving Average (EMA).
Further, \cite{rade2022reducing} shows that HAT achieves the SOTA performance for these extra data.

Table \ref{table-extra} compares ARoW with the exiting algorithms for extra data, which shows that ARoW achieves the state-of-the-art performance when extra data are available even though the margins compared to HAT
are not significant. 
Note that ARoW has advantages other than the high robust accuracies.
For example, ARoW
is easy to implement compared to HAT since HAT requires a pre-trained model.
Moreover, as we will see in Table \ref{table5}, ARoW improves the fairness compared to TRADES while HAT improves the performance with sacrificing fairness.

%-----------------------------------------------
\begin{table*}[ht]
    \caption{\textbf{Modifications of TRADES and ARoW.} We use CIFAR10 dataset and ResNet-18 architecture. More details of hyerparameters are provided in Appendix \ref{app_combine}.}
    %The accuracies are reported for single run.
    \centering
    \begin{tabular}{c|ccc|ccc}
    \toprule
    \multirow{2}{*}{\textbf{Method}} & \multicolumn{3}{c|}{AWP}   & \multicolumn{3}{c}{FAT}  \\
    \cline{2-7}
    & \textbf{Standard}  & \textbf{PGD}$^{20}$ & \textbf{AutoAttack} & \textbf{Standard}  & \textbf{PGD}$^{20}$ & \textbf{AutoAttack} \\
    \hline
    \hline
    TRADES                & 82.10(0.09) & 53.56(0.18) & 49.56(0.23) & 82.96(0.08) & 52.76(0.22) & 49.83(0.28) \\
    ARoW                & \textbf{84.98}(0.11) &\textbf{55.55}(0.15) & \textbf{50.64}(0.18) & \textbf{86.21}(0.06) & \textbf{53.37}(0.20) & \textbf{50.07}(0.17) \\
    \bottomrule
    \end{tabular}
    \label{combination}
    \vskip -0.1in
\end{table*}
%-----------------------------------------------
\begin{table}[ht]
    \caption{\textbf{Class-wise accuracy disparity for CIFAR10}. We report the accuracy (ACC), the worst-class accuracy (WC-Acc) and the standard deviation of class-wise accuracies (SD) for each method.}
    \centering
    \scalebox{0.8}{
    \begin{tabular}{c|ccc|ccc}
    \hline
    \multirow{2}{*}{\textbf{Method}}  & \multicolumn{3}{c|}{\textbf{Standard}} & \multicolumn{3}{c}{$\textbf{PGD}^{10}$} \\
    \cline{2-7}
    & \textbf{Acc} &  \textbf{WC-Acc} & \textbf{SD}  & \textbf{Acc} & \textbf{WC-Acc} & \textbf{SD} \\
    \hline
    \hline
%    TRADES($\lambda=4$) & 87.73 & 70.70 & 8.17 & 57.17 & 26.40 & 16.75 \\
    TRADES & 85.69 & 67.10 & 9.27 & 57.38 & 27.10 & 16.97  \\
    HAT   & 86.74 & 65.40 & 11.12 & 57.92 & 24.20  & 18.26 \\
%    TRADES($\lambda=8$) & 84.94 & 65.90 & 9.58 & 58.01 & 27.30 & 16.92  \\
    ARoW                & \textbf{87.58} & \textbf{74.51} & \textbf{7.11} & \textbf{59.32} & \textbf{31.05} & \textbf{15.67}  \\
    %CoW                 & 88.41 & 72.20 & 7.22 & 58.34 & 26.40 & 17.09 \\
    \hline
    \end{tabular}
    }
    \label{table5}
    \vskip -0.1in
\end{table}
%-----------------------------------------------
%-----------------------------------------------
\subsection{Ablation studies}
\label{ablation}
%-----------------------------------------------
We study the following three issues - (i) the effect of label smoothing to ARoW, 
(ii) the effect of stochastic weighted averaging \cite{izmailov2018averaging},
(iii) the role of the new regularization term in ARoW to improve robustness
and (iV) modifications of ARoW by applying tools
which improve existing adversarial training algorithms.

\subsubsection{Effect of Stochastic Weighting Averaging}
The table presented in Appendix \ref{swa} demonstrates a significant improvement in the performance of ARoW when SWA is applied.
We believe this improvement is primarily due to the adaptive weighted regularization effect of SWA. 
Ensembling methods can improve the performance of models by diversifying them \citep{jantre2022sequential} and SWA can be considered one of the ensembling methods \citep{izmailov2018averaging}. In the case of ARoW, the adaptively weighted regularization term $(1-p_{\bm{\theta}}(y|\widehat{\bm{x}}^{\text{pgd}}))$ diversifies the models for averaging weights, which significantly improves the performance of ARoW.

\subsubsection{Effect of Label Smoothing}

Table \ref{effect-ls} indicates that 
 label smoothing is helpful not only for ARoW but also for TRADES. This would be partly because the regularization terms in ARoW and TRADES
depend on the conditional class probabilities and it is well known that label smoothing is helpful for the calibration of the conditional class probabilities \citep{pereyra2017regularzing}. 

Moreover, the results in Table \ref{effect-ls} imply that label smoothing is not a main reason for
ARoW to outperform TRADES. Even without label smoothing, ARoW is still superior to TRADES (even with the
label smoothing). 
Appendix \ref{ablation:ls} presents the results of an additional experiment to assess
the effect of label smoothing to the performance.
%{\bf The results suggest that label smoothing is helpful unless $\alpha$ is too small.}

\subsubsection{Role of the new regularization term in ARoW}

The regularization term of ARoW puts
more regularization to less robust samples, and thus
we expect that ARoW improves the robustness of less robust samples much. 
To confirm this conjecture, we do a small experiment. 

First, we divide the test data into four groups - least robust, less robust, robust and highly robust
according to the values of
 $p_{\bm{\theta}_{\text{PGD}}}(y_i|\widehat{\bm{x}}_i^{\text{pgd}})$ ($<0.3,$ $0.3 \sim 0.5$,  $0.5 \sim 0.7$ and $>0.7$), where $\bm{\theta}_{\text{PGD}}$ is the parameter learned by
 PGD-AT \citep{madry2018towards}\footnote{We use PGD-AT instead of a standard non-robust
 training algorithm since all samples become least robust for a non-robust prediction model.}.
Then, for each group, we check how many samples become robust for ARoW as well as  TRADES, MART whose results are presented in Tables \ref{effect-regularization-trades} and \ref{effect-regularization-mart}.
Note that ARoW improves the robustness of initially less robust samples  compared with TRADES and MART, respectively.
We believe that this improvement is due to the regularization term in ARoW that enforces more regularization on less robust samples.

%-----------------------------------------------
\subsubsection{Modifications of ARoW}

There are many useful tools which improve existing adversarial training algorithms.
Examples are Adversarial Weight Perturbation (AWP) \citep{wu2020adversarial} and Friendly Adversarial Training (FAT) \citep{zhang2020attacks}. 
AWP is a tool to find a flat minimum of the objective function
and FAT uses early-stopped PGD when generating adversarial examples in the training phase.
Details about AWP and FAT are given in Appendix \ref{app_combine}.

We investigate how ARoW performs  when it is modified by such a tool.
We consider the two  modifications of ARoW - ARoW-AWP and ARoW-FAT, where
ARoW-AWP searches a flat minimum of the ARoW objective function and
ARoW-FAT uses early-stopped PGD in the training phase of ARoW.

Table \ref{combination} compares ARoW-AWP and ARoW-FAT to TRDAES-AWP and TRADES-FAT.
Both of AWP and FAT are helpful for ARoW and TRADES but
ARoW still outperforms TRADES with large margins even after modified by AWP or FAT.

\subsection{Improved Fairness}
\label{fairness}
 
\citet{xu2021to} reports that TRADES \citep{zhang2019theoretically} increases
the variation of the per-class accuracies (accuracy in each class) which
is not desirable in view of fairness. In turn, \citet{xu2021to} proposes the Fair-Robust-Learning (FRL) algorithm to alleviate this problem. Even if
fairness becomes improved, the standard and robust accuracies of FRL 
are worse than TRADES.

In contrast, Table \ref{table5} shows that ARoW improves the fairness
as well as the standard and robust accuracies compared to TRADES.
%More experimental results are provided in Appendix \ref{appE}. 
This desirable property of ARoW can be partly understood as follows.
The main idea of ARoW is to impose more robust regularization
to less robust samples.
In turn, samples in less accurate classes tend to be more vulnerable to adversarial attacks. Thus, ARoW improves the robustness of samples in less accurate classes which results in improved robustness as well as improved generalization for such less accurate classes.
The class-wise accuracies are presented in Appendix \ref{appE}.

%-----------------------------------------------
\section{Conclusion and Future Works}

In this paper, we derived an upper bound of the robust risk and developed a new algorithm for adversarial training called ARoW which minimizes a surrogate version of the derived upper bound.
A novel feature of ARoW is to impose more regularization on less robust samples than TRADES.
The results of numerical experiments shows that ARoW improves the standard and robust accuracies
simultaneously to achieve state-of-the-art performances. In addition, ARoW enhances
the fairness of the prediction model without hampering the accuracies.

 %----------------------------------------------- label smoothing
\begin{table}
    \small
    \caption{\textbf{Comparison of TRADES and ARoW with/without label smoothing.} With WRN-34-10 architecture and CIFAR10 dataset, we use $\lambda=6$ for TRADES while use $\lambda=3$ for ARoW.}
    \centering
    \begin{tabular}{c|ccc}
    \hline
    \textbf{Method} & \textbf{Standard}  & $\textbf{PGD}^{20}$ & \textbf{AutoAttack}  \\
    \hline
    \hline
    TRADES w/o-LS   & 85.86(0.09) & 56.79(0.08) & 54.31(0.08) \\
    TRADES w/-LS   & 86.33(0.08) & 57.45(0.02) & 54.66(0.08) \\
    ARoW w/o-LS    & 86.83(0.16) & 58.34(0.09) & 55.01(0.10) \\
    ARoW w/-LS    & \textbf{87.65}(0.02) & \textbf{58.38}(0.09) & \textbf{55.15}(0.14) \\
    \hline
    \end{tabular}
    \label{effect-ls}
    \vskip -0.1in
\end{table}
%----------------------------------------------- label smoothing

When we developed a computable surrogate of the upper bound of the robust risk in Theorem 1,
we replaced $\mathbbm{1}( F_{\bm{\theta}}(\mathbf{X})
\neq F_{\bm{\theta}}(z(\mathbf{X})))
)$ by $\operatorname{KL}(\mathbf{p}_{\bm{\theta}}(\cdot|\mathbf{X}) || \mathbf{p}_{\bm{\theta}}(\cdot|\widehat{\mathbf{X}}^{\text{pgd}})).$
The KL divergence, however, is not an upper bound of the 0-1 loss and thus our
surrogate is not an upper bound of the robust risk. 
We employed the KL divergence surrogate to make the objective function of ARoW be similar
to that of TRADES. It would be worth pursuing to
devise an alternative surrogate for the 0-1 loss to reduce the gap between the theory and algorithm.

We have seen in Section \ref{fairness} that ARoW improves
fairness as well as accuracies.
The advantage of ARoW in view of fairness
is an unexpected by-product, and it would be interesting to
develop a more principled way of enhancing the fairness further without
hampering the accuracy.

\paragraph{Acknowledgement}
This work was supported by National Research Foundation of Korea (NRF) grant funded by the Korea government (MSIT) (No. 2020R1A2C3A0100355014), Institute of Information \& communications Technology Planning \& Evaluation (IITP) grant funded by the Korea government (MSIT) 
[NO.2022-0-00184, Development and Study of AI Technologies to Inexpensively Conform to Evolving Policy on Ethics].

\bibliography{icml2023_conference}
\bibliographystyle{icml2023}

%%%%%%%%%%%%%%%%%%%%%%%%%%%%%%%%%%%%%%%%%%%%%%%%%%%%%%%%%%%%%%%%%%%%%%%%%%%%%%%
%%%%%%%%%%%%%%%%%%%%%%%%%%%%%%%%%%%%%%%%%%%%%%%%%%%%%%%%%%%%%%%%%%%%%%%%%%%%%%%
% APPENDIX
%%%%%%%%%%%%%%%%%%%%%%%%%%%%%%%%%%%%%%%%%%%%%%%%%%%%%%%%%%%%%%%%%%%%%%%%%%%%%%%
%%%%%%%%%%%%%%%%%%%%%%%%%%%%%%%%%%%%%%%%%%%%%%%%%%%%%%%%%%%%%%%%%%%%%%%%%%%%%%%
\newpage
\appendix
\onecolumn

\section{Proof of Theorem \ref{theorem1}} \label{appA}
\renewcommand{\theequation}{A.\arabic{equation}}
In this section, %we provide the additional theoretical results and proofs. 
%proofs of the main theorem and the alternative upper bound.
we prove Theorem \ref{theorem1}.
The following lemma provides the key inequality for the proof.

\begin{lemma}
\label{lemma2}
For a given score function $f_{\bm{\theta}},$let $z(\cdot)$ be an any measurable mapping from $\mathcal{X}$ to $\mathcal{X}$ satisfying
\begin{equation*}
    z(\bm{x}) \in \underset{\bm{x}' \in \mathcal{B}_{p}(\bm{x}, \varepsilon)}{\operatorname{argmax}} \mathbbm{1} \left( F_{\bm{\theta}}(\bm{x}) \neq F_{\bm{\theta}}(\bm{x}')\right)
\end{equation*}
for every $\bm{x} \in \mathcal{X}$. Then, we have
\begin{align}
    \begin{split}
    \mathbbm{1}\left\{\exists \bm{x}' \in \mathcal{B}_p(\bm{x}, \varepsilon) : F_{\bm{\theta}}(\bm{x})  \neq F_{\bm{\theta}}(\bm{x}'), F_{\bm{\theta}}(\bm{x}) = Y \right\} \\
    \leq \mathbbm{1}\left\{F_{\bm{\theta}}(\bm{x}) \neq F_{\bm{\theta}}(z(\bm{x})), Y \neq F_{\bm{\theta}}(z(\bm{x}))\right\} \label{ineq_insung_multi}
    \end{split}
\end{align}
\end{lemma}
\begin{proof}
The inequality holds obviously if 
$\mathbbm{1}\left\{F_{\bm{\theta}}(\bm{x}) \neq F_{\bm{\theta}}(z(\bm{x})), Y \neq F_{\bm{\theta}}(z(\bm{x}))\right\} = 1$.
Hence, it suffices to show that $\mathbbm{1}\left\{\exists \bm{x}' \in \mathcal{B}_p(\bm{x}, \varepsilon) : F_{\bm{\theta}}(\bm{x})  \neq F_{\bm{\theta}}(\bm{x}'), F_{\bm{\theta}}(\bm{x}) = Y \right\}=0$
when either $F_{\bm{\theta}}(\bm{x}) = F_{\bm{\theta}}(z(\bm{x}))$ or $Y = F_{\bm{\theta}}(z(\bm{x})).$

Suppose $F_{\bm{\theta}}(\bm{x}) = F_{\bm{\theta}}(z(\bm{x})).$
It trivially holds that
$\mathbbm{1} \left( F_{\bm{\theta}}(\bm{x}) \neq F_{\bm{\theta}}(z(\bm{x}))\right) 
\leq \mathbbm{1} \left( F_{\bm{\theta}}(\bm{x}) \neq F_{\bm{\theta}}(\bm{x}')\right)$
for every $\bm{x}' \in \mathcal{X}$ since $\mathbbm{1} \left( F_{\bm{\theta}}(\bm{x}) \neq F_{\bm{\theta}}(z(\bm{x}))\right)=0$ and the equality holds if and only if $F_{\bm{\theta}}(z(\bm{x})) = F_{\bm{\theta}}(\bm{x}')$.
By the definition of $z(\bm{x})$,
the left side of (\ref{ineq_insung_multi}) is $0$ since 
$\mathbbm{1}\left\{\exists \bm{x}' \in \mathcal{B}_p(\bm{x}, \varepsilon) : F_{\bm{\theta}}(\bm{x})  \neq F_{\bm{\theta}}(\bm{x}')\right\} = 0,$
and hence the inequality holds.

Suppose $Y = F_{\bm{\theta}}(z(\bm{x})).$ 
If $F_{\bm{\theta}}(\bm{x}) = Y$ and there exists $\bm{x}'$ in $\mathcal{B}_p(\bm{x}, \varepsilon)$ such that
$F_{\bm{\theta}}(\bm{x}')  \neq F_{\bm{\theta}}(\bm{x})$,
then we have $F_{\bm{\theta}}(\bm{x}') \neq Y = F_{\bm{\theta}}(\bm{x}) = F_{\bm{\theta}}(z(\bm{x}))$. In turn, it implies
$\mathbbm{1} \left( F_{\bm{\theta}}(\bm{x}) \neq F_{\bm{\theta}}(z(\bm{x}))\right) 
< \mathbbm{1} \left( F_{\bm{\theta}}(\bm{x}) \neq F_{\bm{\theta}}(\bm{x}')\right),$ 
which is a contradiction to the definition of $z(\bm{x})$. 
Hence, the left side of (\ref{ineq_insung_multi}) should be $0$, and we complete the proof
of the inequality.
\end{proof}

\multi*
\begin{proof}

Note that $\mathcal{R}_{\text{rob}}({\bm{\theta}}) = \mathcal{R}_{\text{nat}}({\bm{\theta}}) + \mathcal{R}_{\text{bdy}}({\bm{\theta}})$ where
$\mathcal{R}_{\text{nat}}(\bm{\theta}) = \mathbb{E}_{(\mathbf{X}, Y)}\mathbbm{1}\left\{ F_{\bm{\theta}}(\mathbf{X}) \neq Y \right\}$ and 
$\mathcal{R}_{\text{bdy}}(\bm{\theta}) = \mathbb{E}_{(\mathbf{X}, Y)}\mathbbm{1}\left\{\exists \mathbf{X}' \in \mathcal{B}_p(\mathbf{X}, \varepsilon) :F_{\bm{\theta}}(\mathbf{X})\neq F_{\bm{\theta}}(\mathbf{X}')   ,F_{\bm{\theta}}(\mathbf{X}) = Y   \right\}$.

Since
\begin{align*} 
\mathcal{R}_{\text{bdy}}({\bm{\theta}}) &= \mathbb{E}_{(\mathbf{X}, Y)}\mathbbm{1}\left\{\exists \mathbf{X}' \in \mathcal{B}_p(\mathbf{X}, \varepsilon) : F_{\bm{\theta}}(\mathbf{X}) \neq F_{\bm{\theta}}(\mathbf{X}'), F_{\bm{\theta}}(\mathbf{X})=Y \right\} \\ 
& \leq \mathbb{E}_{(\mathbf{X}, Y)}\mathbbm{1}\left\{F_{\bm{\theta}}(\mathbf{X}) \neq F_{\bm{\theta}}(z(\mathbf{X})), Y \neq F_{\bm{\theta}}(z(\mathbf{X}))\right\} (\because \text{\;Lemma\;} \ref{lemma2})\\
& = \mathbb{E}_{(\mathbf{X}, Y)} \mathbbm{1} \left\{ F_{\bm{\theta}}(\mathbf{X}) \neq F_{\bm{\theta}}(z(\mathbf{X})) \right\} \mathbbm{1}\left\{ Y \neq F_{\bm{\theta}}(z(\mathbf{X})) \right\} \\ 
& \leq  \mathbb{E}_{(\mathbf{X}, Y)} \mathbbm{1} \left\{ F_{\bm{\theta}}(\mathbf{X}) \neq F_{\bm{\theta}}(z(\mathbf{X})) \right\} \mathbbm{1} \left\{ p_{\bm{\theta}}(Y | z(\mathbf{X})) < 1/2 \right\},
\end{align*}
the inequality (\ref{arow:ub}) holds.
\end{proof}

\begin{theorem}
\label{thm:cow}
For a given score function $f_{\bm{\theta}}$, let $z(\cdot)$ be an any measurable mapping from $\mathcal{X}$ to $\mathcal{X}$ satisfying
\begin{equation*}
    z(\bm{x}) \in \underset{\bm{x}' \in \mathcal{B}_{p}(\bm{x}, \varepsilon)}{\operatorname{argmax}} \mathbbm{1} \left( F_{\bm{\theta}}(\bm{x}) \neq F_{\bm{\theta}}(\bm{x}')\right).
\end{equation*}
for every $\bm{x} \in \mathcal{X}$. Then, we have
\begin{align}
\label{cow:ub}
    \mathcal{R}_{\text{rob}}(\bm{\theta}) & \leq \mathbb{E}_{(\mathbf{X},Y)} \mathbbm{1}(Y \neq F_{\bm{\theta}}(\mathbf{X})) \nonumber \\
    & + 2 \mathbb{E}_{(\mathbf{X}, Y)}{\mathbbm{1}(F_{\bm{\theta}}(\mathbf{X}) \neq F_{\bm{\theta}}(z(\mathbf{X})))} \cdot p_{\bm{\theta}}(Y|\mathbf{X})
\end{align}
\begin{proof}

Note that $\mathcal{R}_{\text{rob}}({\bm{\theta}}) = \mathcal{R}_{\text{nat}}({\bm{\theta}}) + \mathcal{R}_{\text{bdy}}({\bm{\theta}})$ where
$\mathcal{R}_{\text{nat}}(\bm{\theta}) = \mathbb{E}_{(\mathbf{X}, Y)}\mathbbm{1}\left\{ F_{\bm{\theta}}(\mathbf{X}) \neq Y \right\}$ and 
$\mathcal{R}_{\text{bdy}}(\bm{\theta}) = \mathbb{E}_{(\mathbf{X}, Y)}\mathbbm{1}\left\{\exists \mathbf{X}' \in \mathcal{B}_p(\mathbf{X}, \varepsilon) :F_{\bm{\theta}}(\mathbf{X})\neq F_{\bm{\theta}}(\mathbf{X}')   ,F_{\bm{\theta}}(\mathbf{X}) = Y   \right\}$.

Since
\begin{align*} 
\mathcal{R}_{\text{bdy}}({\bm{\theta}}) &= \mathbb{E}_{(\mathbf{X}, Y)}\mathbbm{1}\left\{\exists \mathbf{X}' \in \mathcal{B}_p(\mathbf{X}, \varepsilon) : F_{\bm{\theta}}(\mathbf{X}) \neq F_{\bm{\theta}}(\mathbf{X}'), F_{\bm{\theta}}(\mathbf{X})=Y \right\} \\  
& \leq \mathbb{E}_{(\mathbf{X}, Y)} \mathbbm{1} \left\{ F_{\bm{\theta}}(\mathbf{X}) \neq F_{\bm{\theta}}(z(\mathbf{X})) \right\} \mathbbm{1}\left\{ Y = F_{\bm{\theta}}(\mathbf{X}) \right\} \\ 
& \leq  \mathbb{E}_{(\mathbf{X}, Y)} \mathbbm{1} \left\{ F_{\bm{\theta}}(\mathbf{X}) \neq F_{\bm{\theta}}(z(\mathbf{X})) \right\} \mathbbm{1}\{p_{\bm{\theta}}(Y|\mathbf{X}) > 1/2 \} \\
& \leq 2 \mathbb{E}_{(\mathbf{X}, Y)} \mathbbm{1} \left\{ F_{\bm{\theta}}(\mathbf{X}) \neq F_{\bm{\theta}}(z(\mathbf{X})) \right\}\cdot p_{\bm{\theta}}(Y|\mathbf{X}), 
\end{align*}
the inequality (\ref{cow:ub}) holds.
\end{proof}

\end{theorem}

\section{Confidence Weighted Regularization (CoW)}
\label{cow}
Motivated from \ref{thm:cow}, we propose the Confidence Weighted Regularization (CoW) which minimizes the following empirical risk:
 \begin{align*}
    \mathcal{R}_{\text{CoW}}(\bm{\theta} ; \left\{(\bm{x}_i, y_i) \right\}_{i=1}^n, \lambda) := \sum\limits_{i=1}^n \bigg\{ \ell^{\text{LS}}(f_{\bm{\theta}}(\bm{x}_i), y_i) + 2 \lambda \cdot \operatorname{KL}(\mathbf{p}_{\bm{\theta}}(\cdot|\bm{x}_i) || \mathbf{p}_{\bm{\theta}}(\cdot|\widehat{\bm{x}}^{\text{pgd}}_i)) \cdot p_{\bm{\theta}}(y_i|\bm{x}_i)\bigg\}
    \label{cow:surrogate}.
\end{align*}

%CoW imposes more regularization on correctly classified sample. 
%It is improper to put more regularization on misclassified samples since 
%the clean sample itself is adversarial examples for true label. 
%That is, minimizing $\operatorname{KL}(\bm{p}_{\bm{\theta}}(\cdot|\bm{x}) || \bm{p}_{\bm{\theta}}(\cdot|\widehat{\bm{x}}))$
%is not appropriate for misclassifed sample.

\subsection{Experimental Comparison of CoW to MART}
\label{compare-cow-mart}
We compare the CoW and MART for four attack methods included in AutoAttack \cite{croce2020reliable}.
CoW outperforms MART both on standard accuracies and robust accuracies except for $\text{PGD}^{20}$.

\begin{table}[ht]
    \caption{\textbf{Comparison of MART and CoW}. We compare the robustness of MART \citep{wang2020improving} and ARoW against the four attacks used  in AutoAttack on CIFAR10. The results are based on WRN-34-10. We set $
    \lambda=4$ for CoW.}
    \centering
    \begin{tabular}{c|c|cccc}
    \hline
    \textbf{Method} & \textbf{Standard}  & \textbf{APGD} & \textbf{APGD-DLR}  & \textbf{FAB} & \textbf{SQUARE} \\
    \hline
    \multirow{1}{*}{\text{MART}}
    & 83.17  & \textbf{56.30} & 51.87 & 51.28  & 58.59 \\
    \multirow{1}{*}{\text{CoW}}
    & \textbf{88.53} & 56.15 & \textbf{54.79} & \textbf{56.67} & \text{61.88}  \\
    %\multirow{1}{*}{\text{ARoW}}
    %& 87.65 & \textbf{58.38} & \textbf{56.37} & \textbf{55.17} & \textbf{56.69} & \textbf{63.50} \\
    \hline
    \end{tabular}
    %\label{compare-cow-mart}
\end{table}

We divide the test data into four groups - least correct, less correct, correct and highly correct
according to the values of 
$p_{\bm{\theta}_{\text{PGD}}}(y|\bm{x})$ ($<0.3,$ $0.3 \sim 0.5$,  $0.5 \sim 0.7$ and $>0.7$), where $\bm{\theta}_{\text{PGD}}$ is the parameter learned by PGD-AT \citep{madry2018towards}.
Note that CoW improves the robustness of correct and highly correct samples compared with MART.
We believe that this improvement is due to the regularization term in CoW that enforces more regularization on correct samples.

%-----------------------------------------------
 \begin{table}[ht]
    \caption{\textbf{Comparing of CoW to MART on sample's robustness.} \# $\textbf{Rob}_{\text{MART}}$ and \# $\textbf{Rob}_{\text{CoW}}$ represent the number of samples which robust to MART and CoW, respectively. \textbf{Diff.} and \textbf{Rate of Impro.} denote (\# $\textbf{Rob}_{\text{CoW}}$ - \# $\textbf{Rob}_{\text{MART}}$ and \textbf{Diff.} / \# $\textbf{Rob}_{\text{MART}}$). The autoattack is used for evaluating the robustness because of gradient masking.}
    \centering
    \begin{tabular}{r|cccc}
    \hline
    Sample's Correctness & \# $ \textbf{Rob}_{\text{MART}}$ & \# $\textbf{Rob}_{\text{CoW}}$ & \textbf{Diff.} & \textbf{Rate of Impro. (\%) } \\
    \hline
    \hline
    Least Correct       & 0 & 3 & 3 & - \\
    Less Correct        & 78 & 59 & -19 & -24.05 \\
    Correct             & 322 & 346 & 24 & 7.45 \\
    Highly Correct      & 4958 & 5072 & \textbf{114} & 2.30 \\
    \hline
    \end{tabular}
    \label{effect-regularization-cow-mart}
    \vskip -0.1in
\end{table}
%-----------------------------------------------

\section{Detailed settings for the experiments with benchmark datasets}
\label{appB}

\subsection{Experimental Setup}
For CIFAR10, SVHN and FMNIST datasets, input images are normalized into [0, 1]. 
Random crop and random horizontal flip with probability 0.5 are used for CIFAR10 while only random horizontal flip with probability 0.5 is applied for SVHN.
For FMNIST, augmentation is not used.

For generating adversarial examples in the training phase, PGD$^{10}$  with random initial, 
$p=\infty$, $\varepsilon = 8/255$ and $\nu = 2/255$ is used, where PGD$^T$ is the output of the PGD algorithm (\ref{pgd}) with $T$ iterations.
For training prediction models,
the SGD with momentum $0.9$, weight decay  $5 \times 10^{-4}$, the initial learning rate of 0.1 and batch size of 128 is used and the learning rate
is reduced by a factor of 10 at 60 and 90 epochs.
Stochastic weighting average (SWA) \citep{izmailov2018averaging} is employed after 50-epochs for preventing from robust overfitting \citep{rice2020overfitting} as \citet{chen2021robust} does.

For evaluating the robustness in the test phase, PGD$^{20}$ and AutoAttack are used for adversarial attacks, where AutoAttack consists of three white box attacks - APGD and APGD-DLR in \cite{croce2020reliable} and FAB in \cite{croce2020minimally} and one black box attack -  Square Attack \citep{andriushchenko2020square}. To the best of our knowledge, AutoAttack is the strongest attack.
%The effect of SWA is presented in Appendix \ref{swa}.
%The regularization parameters are set to be the ones given in the corresponding articles.
The final model is set to be the best model against PGD$^{10}$ on the test data among those obtained until 120 epochs. 
%Detailed settings used in each experiment are summarized in Appendix \ref{appB}.

\subsection{Hyperparameter setting}

\begin{table}[ht]
    \centering
    \caption{\textbf{Selected hyperparameters.} Hyperparameters used in the numerical studies in Section \ref{sec4_1}.}
    \begin{tabular}{c|c|c|ccccc}
    \hline
     Dataset& Model & $\textbf{Method}$ & $ \lambda$ & $\gamma$  & Weight Decay & $\alpha$ & SWA  \\
    %$1 / \lambda$
    \hline
    \hline
    %\cmidrule(r){1-6}
    \multirow{6}{*}{\text{CIFAR10}} & \multirow{6}{*}{\text{WRN-34-10}} & TRADES  & 6 & -    & $5e^{-4}$ & -   & o \\
    & & HAT     & 4 & 0.25 & $5e^{-4}$ & -   &  o \\
    & & MART    & 5 & -    & $2e^{-4}$ & -   &  o \\
    & & PGD-AT     & - & - & $5e^{-4}$ & -   &  o \\
    & & GAIR-AT     & - & - & $5e^{-4}$ & -   &  o \\
    & &  ARoW    & 3 & -    & $5e^{-4}$ & 0.2 &  o \\
    \cmidrule(r){1-8}
    \multirow{6}{*}{\text{CIFAR100}} & \multirow{6}{*}{\text{WRN-34-10}} &TRADES          & 6 & -    & $5e^{-4}$ & -   &  o \\
    & & HAT             & 4 & 0.5 & $5e^{-4}$ & -   &  o \\
    & & MART            & 4 & -    & $5e^{-4}$ & -   &  o \\
    & & PGD-AT     & - & - & $5e^{-4}$ & -   &  o \\
    & & GAIR-AT     & - & - & $5e^{-4}$ & -   &  o \\
    & & ARoW            & 4 & -   & $5e^{-4}$ & 0.2 & o \\
    \hline
    \multirow{6}{*}{\text{SVHN}} & \multirow{6}{*}{\text{ResNet-18}} &TRADES   & 6 & -    & $5e^{-4}$ & -   &  x\\
    & & HAT      & 4 & 0.5  & $5e^{-4}$ & -   &  x\\
    & & MART     & 4 & -    & $5e^{-4}$ & -   &  x\\
    & & PGD-AT     & - & - & $5e^{-4}$ & -   &  x \\
    & & GAIR-AT     & - & - & $5e^{-4}$ & -   &  x \\
    & & ARoW     & 3 & -    & $5e^{-4}$ & 0.2 & x\\
    \hline
    \multirow{6}{*}{\text{FMNIST}} & \multirow{6}{*}{\text{ResNet-18}} &TRADES   & 6 & -    & $5e^{-4}$ & -   &  x\\
    & & HAT      & 5 & 0.15 & $5e^{-4}$ & -   &  x\\
    & & MART     & 4 & -    & $5e^{-4}$ & -   &  x\\
    & & PGD-AT     & - & - & $5e^{-4}$ & -   &  x \\
    & & GAIR-AT     & - & - & $5e^{-4}$ & -   &  x \\
    & & ARoW     & 6 & -   & $5e^{-4}$ & 0.25 & x\\
    \hline
    \end{tabular}
    \label{table_hyper}
\end{table}

Table \ref{table_hyper} presents the hyperparameters used on our experiments.
Most of the hyperparameters are set to be the ones used in the previous studies.
The weight decay parameter is set to be $5e^{-4}$ in most experiments, which is the well-known optimal value.
We use stochastic weight averaging (SWA) for CIFAR10 and CIFAR100.
Only for MART \citep{wang2020improving} with WRN-34-10, we use weight decay $2e^{-4}$ 
as \cite{wang2020improving} did since MART works poorly with $5e^{-4}$ with SWA.

%Note that SWA is not used in the experiments of \citet{wang2020improving}, 
%and we confirm that SWA is not helpful for MART.
%Effects of SWA for all methods are provided in Appendix \ref{swa}.
 
%The other regularization parameters $\alpha$ and $\tau$ in ARoW is selected so that the robust accuracy against $\text{PGD}^{10}$ is maximized for fixed $\lambda=4$.

\section{Checking the Gradient Masking}
\label{gradient-masking}

\begin{table}[H]
    \caption{\textbf{Comparison of GAIR-AT and ARoW}. We compare the robustness of GAIR-AT \citep{zhang2021geometry} and ARoW against the four attacks used  in AutoAttack on CIFAR10. The results are based on WRN-34-10. We set $\lambda=3$ for ARoW.}
    \centering
    \begin{tabular}{c|c|c|cccc}
    \hline
    \textbf{Method} & \textbf{Standard}  & \textbf{PGD} & \textbf{APGD} & \textbf{APGD-DLR}  & \textbf{FAB} & \textbf{SQUARE} \\
    \hline
    \multirow{1}{*}{\text{GAIR-AT}}
    & 85.44(0.17) & \textbf{67.27}(0.07) & \textbf{63.14}(0.16) & 46.48(0.07) &  49.35(0.05) & 55.19(0.16) \\
    \multirow{1}{*}{\text{ARoW}}
    & \textbf{87.65}(0.02) & 58.38(0.09) & 56.07(0.14) & \textbf{55.17}(0.11) & \textbf{56.69}(0.17) & \textbf{63.50}(0.08) \\
    \hline
    \end{tabular}
    \label{gradient-masking_table}
\end{table}

\textit{Gradient masking} \citep{papernot2016science, papernot2017practical} is  the case that
the gradient of the loss for a given non-robust datum 
is almost zero (i.e. $\nabla_{\bm{x}}\ell_{\text{ce}}(f_{\bm{\theta}}(\bm{x}), y) \approx \bm{0}$).
In this case, PGD cannot generate an adversarial example. 
We can check the ocuurence of gradient masking when
a prediction model is robust to the PGD attack but not robust to
attacks such as FAB \citep{croce2020minimally}, APGD-DLR \citep{croce2020reliable} and SQUARE \citep{andriushchenko2020square}.
%which do not use the gradient of the loss.

%In most cases, a model robust to white-box attack is more robust to black-box attack. But if gradient masking  occurs, gradient based white-box attack like PGD (\ref{pgd}) is not effective. 
In Table \ref{gradient-masking_table}, the robustness of GAIR-AT becomes worse much for the three attacks in AutoAttack except APGD \citep{croce2020reliable} while the robustness of ARoW remains stable regardless of the adversarial attacks. 
Since APGD uses the gradient of the loss, 
this observation implies that the gradient masking occurs in GAIR-AT while it does not in ARoW.

Better performance of GAIR-AT
for PGD$^{20}$ attack in Table \ref{table:compare} is not because
GAIR-AT is robust to adversarial attacks but because adversarial examples obtained by PGD are close to clean samples. 
This claim is supported by the fact
that GAIR-AT performs poorly for AutoAttack
while it is still robust to other PGD-based adversarial attacks. Moreover,
gradient masking for GAIR-AT is already reported by \citet{hitaj2021evaluating}.

\section{Detailed setting for the experiments with extra data}
\label{app:extra}

\begin{table}[H]
    \centering
    \caption{\textbf{Selected hyperparameters.} Hyperparameters used in the numerical studies in Section \ref{sec4_3}. We do not employ cutmix augmentation \citep{yun2019cutmix} as does in \cite{rade2022reducing}.}
    \begin{tabular}{c|c|cccccc}
    \hline
     Model & $\textbf{Method}$ & $ \lambda$ & $\gamma$  & Weight Decay & $\alpha$ & EMA & SiLU  \\
    %$1 / \lambda$
    \hline
    \hline
    %\cmidrule(r){1-6}
    \multirow{4}{*}{\text{WRN-28-10}} & \cite{carmon2019unlabeled}  & 6 & - & $5e^{-4}$ & - & x & x \\
    & \cite{rebuffi2021data} & 6 & - & $5e^{-4}$ & - & o & o \\
    & HAT & 4 & 0.25 & $5e^{-4}$ & - & o & o \\
    & ARoW & 3.5 & - & $5e^{-4}$ & 0.2 & o & o \\
    \hline
    \multirow{4}{*}{\text{ResNet-18}} & \cite{carmon2019unlabeled}  & 6 & - & $5e^{-4}$ & - & x & x \\
    & \cite{rebuffi2021data} & 6 & - & $5e^{-4}$ & - & o & o \\
    & HAT & 4 & 0.25 & $5e^{-4}$ & - & o & o \\
    & ARoW & 3.5 & - & $5e^{-4}$ & 0.2 & o & o \\
    \hline
    \end{tabular}
    \label{table_hyper:extra}
\end{table}

In Section \ref{sec4_3}, we presented the results of ARoW on CIFAR10 with extra unlabeled data used in \citet{carmon2019unlabeled} and \citet{rebuffi2021data}.
In this section, we provide experimental details.

\citet{rebuffi2021data} use the SiLU activation function and exponential model averaging (EMA) based on TRADES.
For HAT \citep{rade2022reducing} and ARoW, we use the SiLU activation function and exponential model averaging (EMA) with weight decay factor 0.995 as is done in \citet{rebuffi2021data}.
The cosine annealing learning rate scheduler \citep{loshchilov2017sgdr} is used with the batch size 512. The final model is set to be the best model against PGD$^{10}$ on the test data among those obtained until 500 epochs.
%\textcolor{red}{If time rest, simple discription about RST and Rebiffi et al.}

\section{Additional Results with Extra Data}

\subsection{Additional Results with Extra Data}

In the main manuscript, we use architecture of ResNet18, while \citet{rade2022reducing} use PreAct-ResNet18.
For better comparison, we conduct an additional experiment with extra data where the same architecture - PreaAct-ResNet18 is used. In addition, we set batch size to 1024 which is used in \citet{rade2022reducing}.
Table \ref{tab:ex} shows that ARoW outperforms HAT both on standard accuracy($+0.29\%$) and robust accuracy($+0.11\%$) against autoattack.

%We will add results with DDPM extra data on PreActResNet-18 by the end of discussion period.

\begin{table}[H]
        \caption{\textbf{Performance with extra data (Carmon et al.) on CIFAR10.} We brought the values in the paper as reported in \citet{rade2022reducing}.}\label{tab:ex}
        \centering
        \begin{tabular}{c|cc}
        \hline
        \textbf{Method} & \textbf{Standard} & \textbf{AutoAttack}  \\
        \hline
        \hline
        HAT                        & 89.02 & 57.67 \\
        %ARoW(batch=512)            & 89.52 & 57.72 \\
        ARoW           & 89.31 & 57.78 \\
        \hline
        \end{tabular}
        %\label{table1}
\end{table}

\section{Ablation study}
\label{appD}
\subsection{The performance on CIFAR10 - ResNet18}
\begin{table}[H]
    \centering
    \caption{\textbf{Performance on CIFAR10 with ResNet18.} We conduct the experiment three times with different seeds and present the averages of the accuracies with the standard errors in the brackets. `w/o' stands for `without'.}
    \begin{tabular}{c|ccc}
    \hline
    \textbf{Method} & \textbf{Standard}  & $\textbf{PGD}^{20}$  & \textbf{AutoAttack}  \\
    \hline
    \hline
    PGD-AT                      & 82.42(0.05) & 53.48(0.11) & 49.30(0.07) \\
    GAIR-AT                     & 81.09(0.12) & \textcolor{red}{64.89}(0.04) & \textcolor{red}{41.35}(0.16) \\
    TRADES                      & 82.41(0.07) & 52.68(0.22) & 49.63(0.25) \\
    HAT                         & 83.05(0.03) & 52.91(0.08) & 49.60(0.02) \\
    MART                        & 74.87(0.95) & 53.68(0.30) & 49.61(0.24) \\
    ARoW                        & \textbf{82.53}(0.13) & \textbf{55.08}(0.16) & \textbf{51.33}(0.18) \\
    \hline
    \end{tabular}
\end{table}

\subsection{The trade-off due to the choice of $\lambda$}
\label{ablation:lambda}

Table \ref{ablation_lambda} presents the trade-off between the generalization and robustness accuracies of ARoW on CIFAR10 due to the choice of $\lambda,$
where ResNet18 is used. The trade-off is obviously observed.
\begin{table}[H]
    \centering
    \caption{\textbf{Standard and robust accuracies of ARoW on CIFAR10 for varying $\lambda$.}}
    \begin{tabular}{c|ccc}
    \hline
    $\lambda$  & $\textbf{Standard}$ & $\textbf{PGD}^{20}$ & $\textbf{AutoAttack}$  \\
    \hline
    \hline
    TRADES($\lambda = 6$) & 82.41 & 52.68 & 49.63 \\
    \hline
    ARoW($\lambda = 2.5$) & 85.30 & 53.80 & 49.66 \\
    ARoW($\lambda = 3.0$) & 84.65 & 54.23 & 50.11 \\
    ARoW($\lambda = 3.5$) & 83.86 & 54.13 & 50.15 \\
    ARoW($\lambda = 4.0$) & 83.73 & 54.20 & 50.55 \\
    ARoW($\lambda = 4.5$) & 82.97 & 54.69 & 50.83 \\
    ARoW($\lambda = 5.0$) & 82.53 & 55.08 & 51.33 \\
    \hline
    \end{tabular}
    \label{ablation_lambda}
\end{table}

\subsection{The effect of label smoothing}
\label{ablation:ls}

Table \ref{ablation_alpha} presents the standard and robust accuracies
of ARoW on CIFAR10 for various values of the smoothing parameter $\alpha$ in the label smoothing where
the regularization parameter $\lambda$ is fixed at 3 and ResNet18 is used.

\begin{table}[H]
    \centering
    \caption{\textbf{Standard and robust accuracies of ARoW on CIFAR10 for varying $\alpha$.}}
    \begin{tabular}{c|ccc}
    \hline
    $\alpha$  & $\textbf{Standard}$ & $\textbf{PGD}^{20}$ & $\textbf{AutoAttack}$  \\
    \hline
    \hline
    0.05 & 83.54 & 53.10 & 49.88 \\
    0.10 & 84.10 & 53.29 & 49.75 \\
    0.15 & 84.36 & 53.56 & 49.67 \\
    0.20 & 84.52 & 53.68 & 49.96 \\
    0.25 & 84.48 & 53.53 & 49.93 \\
    0.30 & 84.55 & 53.53 & 49.89 \\
    0.35 & 84.66 & 54.19 & 50.03 \\
    0.40 & 84.65 & 54.23 & 50.11 \\
    \hline
    \end{tabular}
    \label{ablation_alpha}
\end{table}

\subsection{Effect of Stochastic Weight Averaging (SWA)}
\label{swa}

We compare the standard and robust accuracies of the adversarial training algorithms with and
without SWA whose results are summarized in Table \ref{table_swa}.
SWA improves the accuracies for all the algorithms except MART. 
Without SWA, ARoW is competitive to HAT, which is known to be the SOTA method. However,
ARoW dominates HAT when SWA is applied.

\begin{table}[H]
    \centering
    \caption{\textbf{Effects of SWA on CIFAR10 with WideResNet 34-10.} We conduct the experiment three times with different seeds and present the averages of the accuracies with the standard errors in the brackets. `w/o' stands for `without'.}
    \begin{tabular}{c|c|ccc}
    \hline
    & \textbf{Method} & \textbf{Standard}  & $\textbf{PGD}^{20}$  & \textbf{AutoAttack}  \\
    \hline
    \hline
    \multirow{5}{*}{\text{SWA}} 
    & TRADES        & 85.86(0.09) & 56.79(0.08) & 54.31(0.08) \\
    & HAT           & 86.98(0.10) & 56.81(0.17) & 54.63(0.07) \\
    & MART          & 78.41(0.07) & 56.04(0.09) & 48.94(0.09) \\
    & PGD-AT        & 87.02(0.20) & 57.50(0.12) & 53.98(0.14) \\
    %& GAIR-AT       & 85.44(0.10) & \textcolor{red}{67.27}(0.07) & \textcolor{red}{46.41}(0.07) \\
    %& CoW           & 88.20(0.09) & 57.33(0.05) & 54.63(0.12) \\
    & ARoW          & \textbf{87.59}(0.02) & \textbf{58.61}(0.09) & \textbf{55.21}(0.14) \\
    %& ARoW-HAT      & \textbf{87.77}(0.03) & 58.54(0.11) & 55.09(0.08) \\
    \cmidrule(r){1-5}
    \multirow{5}{*}{\text{w/o-SWA}}
    & TRADES        & 85.48(0.12) & 56.06(0.08) & 53.16(0.17)\\
    & HAT           & 87.53(0.02) & 56.41(0.09) & \textbf{53.38}(0.10)\\
    & MART          & 84.69(0.18) &  55.67(0.13) & 50.95(0.09)\\
    & PGD-AT        & 86.88(0.09) & 54.15(0.16) & 51.35(0.14)\\
    %& GAIR-AT       & 84.49(0.06) & \textcolor{red}{62.11}(0.12) & \textcolor{red}{38.48} (0.36) \\
    %& CoW           & 86.94(0.08) & 56.19(0.13) & 53.39(0.08) \\
    & ARoW          & \textbf{87.60}(0.02) & \textbf{56.47}(0.10) & 52.95(0.06) \\
    %& ARoW-HAT      & \textbf{87.90}(0.05) & \textbf{57.28}(0.08) & \textbf{53.56}(0.05) \\
    \hline
    \end{tabular}
    \label{table_swa}
\end{table}

\subsection{Various perturbation budget $\varepsilon$}

Table \ref{table_varep_train} and \ref{table_varep_test} show the performance of various perturbation budget $\varepsilon$ for train and test phases, respectively.
The regularization parameters of this studies are 3.5 and 6 for ARoW and TRADES, respectively.
We observe that ARoW outperforms TRADES in all cases.

\begin{table}[H]
    \centering
    \caption{\textbf{Performance of various train perturbation budget $\varepsilon$ on CIFAR10 with ResNet-18.} We train models using ARoW and TRADES with varying $\varepsilon$ and evaluate the robustness with same $\varepsilon=8$.}
    \begin{tabular}{c|c|ccc}
    \hline
    $\varepsilon$ for training & \textbf{Method} & \textbf{Standard}  & $\textbf{PGD}^{20}$  & \textbf{AutoAttack}  \\
    \hline
    \hline
    \multirow{2}{*}{\text{4}} 
    & ARoW            & 89.45 & 72.98 & 71.99 \\
    & TRADES          & 88.30 & 72.22 & 71.29 \\
    \cmidrule(r){1-5}
    \multirow{2}{*}{\text{6}} 
    & ARoW            & 86.40 & 62.84 & 60.33 \\
    & TRADES          & 85.13 & 62.05 & 59.91 \\
    \cmidrule(r){1-5}
    \multirow{2}{*}{\text{8}} 
    & ARoW            & 83.34 & 53.93 & 50.37 \\
    & TRADES          & 82.26 & 52.18 & 49.13 \\
    \cmidrule(r){1-5}
    \multirow{2}{*}{\text{10}} 
    & ARoW            & 81.36 & 45.09 & 40.41 \\
    & TRADES          & 80.09 & 42.75 & 38.47 \\
    \cmidrule(r){1-5}
    \multirow{2}{*}{\text{12}} 
    & ARoW            & 80.03 & 37.87 & 32.14 \\
    & TRADES          & 76.49 & 36.68 & 31.60 \\
    \hline
    \end{tabular}
    \label{table_varep_train}
\end{table}

\begin{table}[H]
    \centering
    \caption{\textbf{Performance of various test perturbation budget $\varepsilon$ on CIFAR10 with ResNet-18.} We train models using ARoW and TRADES with $\varepsilon=8$ and evaluate the performance with varying $\varepsilon$.}
    \begin{tabular}{c|c|ccc}
    \hline
    $\varepsilon$ for test & \textbf{Method} & \textbf{Standard}  & $\textbf{PGD}^{20}$  & \textbf{AutoAttack}  \\
    \hline
    \hline
    \multirow{2}{*}{\text{4}} 
    & ARoW            & 83.34 & 70.61 & 69.02 \\
    & TRADES          & 82.26 & 68.50 & 67.17 \\
    \cmidrule(r){1-5}
    \multirow{2}{*}{\text{6}} 
    & ARoW            & 83.34 & 62.50 & 59.87 \\
    & TRADES          & 82.26 & 61.11 & 58.66 \\
    \cmidrule(r){1-5}
    \multirow{2}{*}{\text{8}} 
    & ARoW            & 83.34 & 53.93 & 50.37  \\
    & TRADES          & 82.26 & 52.18 & 49.13  \\
    \cmidrule(r){1-5}
    \multirow{2}{*}{\text{10}} 
    & ARoW            & 83.34 & 45.13 & 41.01 \\
    & TRADES          & 82.26 & 43.99 & 40.25 \\
    \cmidrule(r){1-5}
    \multirow{2}{*}{\text{12}} 
    & ARoW            & 83.34 & 37.10 & 32.67 \\
    & TRADES          & 82.26 & 36.08 & 32.13 \\
    \hline
    \end{tabular}
    \label{table_varep_test}
\end{table}

\subsection{AWP and FAT}
\label{app_combine}
\subsubsection{Adversarial Weight Perturbation (AWP)}
For a given objective function of the adversarial training, AWP \citep{wu2020adversarial} tries to find a flat minimum in the parameter space. \cite{wu2020adversarial} proposes TRADES-AWP, which minimizes
\begin{align*}
    \min _{\bm{\theta}} \max _{\left\|\bm{\delta}_l\right\| \leq \gamma\left\|\bm{\theta}_l\right\|} \frac{1}{n} \sum_{i=1}^n \left\{ \ell_{\text{ce}}(f_{\bm{\theta}+\bm{\delta}}(\bm{x}_i), y_i) + \lambda \cdot \operatorname{KL} (\mathbf{p}_{\bm{\theta}+\bm{\delta}}(\cdot|\bm{x}_i)\lVert \mathbf{p}_{{\bm{\theta}+\bm{\delta}}}(\cdot|\widehat{\bm{x}}^{\text{pgd}}_i))\right\},
\end{align*}
where  $\bm{\theta}_l$ is the weight vector of $l$-th layer
and $\gamma$ is the weight perturbation size.
Inspired by TRADES-AWP, we propose ARoW-AWP which minimizes
\begin{align*}
    \min _{\bm{\theta}} \max _{\left\|\bm{\delta}_l\right\| \leq \gamma\left\|\bm{\theta}_l\right\|} \frac{1}{n} \sum_{i=1}^n \Big\{ & \ell_{\text{ce}}(f_{\bm{\theta}+\bm{\delta}}(\bm{x}_i), y_i) \\
    & +2 \lambda \cdot \operatorname{KL} (\mathbf{p}_{\bm{\theta}+\bm{\delta}}(\cdot|\bm{x}_i)\lVert \mathbf{p}_{{\bm{\theta}+\bm{\delta}}}(\cdot|\widehat{\bm{x}}^{\text{pgd}}_i))
    \cdot  (1 - p_{\bm{\theta}}(y_i | \widehat{\bm{x}}^{\text{pgd}}_i) ) \Big\}.
\end{align*}

In our experiment, we set $\gamma$ to be 0.005 which is the value used in \cite{wu2020adversarial} and do not use SWA as did in original paper.

\subsubsection{Friendly Adversarial Training (FAT)}
\citet{zhang2020attacks} suggests early-stopped PGD which uses a data-adaptive iterations
of PGD when an adversarial example is generated. TRADES-FAT, which uses the early-stopped PGD
in TRADES, minimizes
\begin{equation*}
    \sum\limits_{i=1}^n \ell_{\text{ce}}(f_{\bm{\theta}}(\bm{x}_i), y_i) + \lambda \cdot\operatorname{KL} (\mathbf{p}_{\theta }(\cdot|\bm{x}_i)\lVert \mathbf{p}_{\theta }(\cdot|\widehat{\bm{x}}^{(t_i)}_{i}))
\end{equation*}
where $t_i =  \min  \left\{ \min \{t : F_{\bm{\theta}}(\widehat{\bm{x}}_i^{(t)}) \neq y_i \} + K, T \right\}.$ Here, $T$ is the maximum iterations of PGD.

We propose an adversarial training algorithm ARoW-FAT by combining ARoW and early-stopped PGD. ARoW-FAT minimizes the following regularized empirical risk:
\begin{equation*}
\label{arow-fat}
    \sum\limits_{i=1}^n \left\{ \ell^{\text{LS}}_{\alpha}(f_{\theta}(\bm{x}_i), y_i) +  2 \lambda \cdot \operatorname{KL} (\mathbf{p}_{\theta }(\cdot|\bm{x}_i)\lVert \mathbf{p}_{\theta }(\cdot|\widehat{\bm{x}}^{(t_i)}_i)) \cdot  (1 - p_{\bm{\theta}}(y_i | \widehat{\bm{x}}^{(t_i)}_i)) \right\}.
\end{equation*}

In the experiments, we set $K$ to be 2,  which is the value used in \cite{zhang2020attacks}. 

\section{Improved fairness}
\label{appE}

Table \ref{table5} shows that ARoW improves the fairness in terms of class-wise accuracies.
the worst-class accuracy (WC-Acc) and standard deviation of class-wise accracies (SD) are defined by $\text{WC-Acc} = \underset{c}{\min}\; \text{Acc}(c)$ and $\text{SD} = \sqrt{\dfrac{1}{C}\sum\limits_{c=1}^C (\text{Acc}(c)-\bar{\text{Acc}}})^2$
where $\text{Acc}(c)$ is the accuracy of class $c$ and $\bar{\text{Acc}}$ is the mean of class-wise accuracies.

\begin{table}[H]
    \caption{\textbf{Comparison of per-class robustness and generalization of TRADES and ARoW.} $\textbf{Rob}_{\text{TRADES}}$ and $\textbf{Rob}_{\text{ARoW}}$ are the robust accuracies against $\text{PGD}^{20}$ of TRADES and ARoW, respectively. $\textbf{Stand}_{\text{TRADES}}$ and $\textbf{Stand}_{\text{ARoW}}$ are the standard accuracies.}
    \centering
    \begin{tabular}{c|cccc}
    \hline
    \textbf{Class} & $\textbf{Rob}_{\text{TRADES}}$ & $\textbf{Rob}_{\text{ARoW}}$ &  $\textbf{Stand}_{\text{TRADES}}$ &  $\textbf{Stand}_{\text{ARoW}}$ \\
    \hline
    \hline
    0(Airplane)    & 64.8 & 66.7 & 88.3 & 91.6 \\
    1(Automobile)  & 77.5 & 77.5 & 93.7 & 95.3 \\
    2(Bird)        & 38.5 & 43.1 & 72.5 & 80.6 \\
    3(Cat)         & 26.1 & 30.2 & 65.9 & 75.1 \\
    4(Deer)        & 35.6 & 40.3 & 83.4 & 87.5 \\
    5(Dog)         & 48.6 & 47.2 & 76.0 & 79.3 \\
    6(Frog)        & 67.8 & 63.6 & 94.2 & 95.2 \\
    7(Horse)       & 69.7 & 69.3 & 91.0 & 92.7 \\
    8(Ship)        & 62.3 & 70.1 & 90.9 & 94.9 \\
    9(Truck)       & 75.3 & 76.3 & 93.5 & 93.5 \\
    \hline
    \end{tabular}
    \label{per-class-comparison}
\end{table}

In Table \ref{per-class-comparison}, we present the per-class robust and standard accuracies of the prediction models trained by TRADES and ARoW. 
We can see that ARoW is highly effective 
for classes difficult to be classified such as Bird, Cat, Deer and Dog.
For such classes, ARoW improves much not only the standard accuracies but
also the robust accuracies.
For example, in the class `Cat', which is the most difficult class
(the lowest standard accuarcy for TRADES and ARoW),
the robustness and generalization are improved by 4.1 percentage point $(26.1\% \rightarrow 30.2\%$) and 9.2 percentage point $(65.9\% \rightarrow 75.1\%)$ 
by ARoW compared with TRADES, respectively. This desirable results would be mainly due to the new regularization term in ARoW. Usually, difficult classes are less robust
to adversarial attacks. By putting more regularization on less robust classes, ARoW improves the accuracies of less robust classes more.  

\end{document}